\documentclass[letterpaper, 11pt]{article}

\usepackage[utf8]{inputenc}
\usepackage{geometry}
\usepackage[utf8]{inputenc}
\usepackage[colorlinks=true,linkcolor=blue,urlcolor=black,bookmarksopen=true, bookmarksopenlevel=1]{hyperref}
\usepackage{bookmark}
\usepackage{url}            
\usepackage{booktabs}       
\usepackage{amsfonts,enumitem}       
\usepackage{nicefrac}       
\usepackage{microtype}      

\usepackage{algorithmic}
\usepackage[linesnumbered,ruled,vlined, algo2e, resetcount]{algorithm2e}

\usepackage{graphicx}
\usepackage{subfigure}
\usepackage{booktabs} 
\usepackage{amsmath}
\usepackage{amsthm}
\usepackage{float}
\usepackage{bm}
\usepackage{amssymb}
\usepackage{bbm}
\usepackage{color}
\usepackage{mathtools}
\usepackage[super]{nth}
\usepackage{etoolbox}
\usepackage{adjustbox}
\usepackage{enumitem}

\def\reals{{\mathbb R}}
\def\prob{{\mathbb P}}

\def\cN{{\cal N}}
\def\eps{\varepsilon}
\def\E{\mathbb E}
\def\F{\mathbf F}
\def\x{\mathbf{x}}

\newtheorem{propo}{Proposition}[section]
\newtheorem{lemma}[propo]{Lemma}
\newtheorem{definition}[propo]{Definition}
\newtheorem{coro}[propo]{Corollary}
\newtheorem{thm}{Theorem}
\newtheorem{asmp}{Assumption}
\newtheorem{theorem}[propo]{Theorem}

\newtheorem{remark}[propo]{Remark}

\newtheorem{example}[propo]{Example}
\newcommand{\ip}[2]{\left\langle #1, #2 \right \rangle}
\DeclareMathOperator{\Tr}{Tr}
\def\ind{{\mathbb I}}

\def\reals{{\mathbb R}}
\def\prob{{\mathbb P}}

\def\cN{{\cal N}}
\def\eps{\varepsilon}

\def\E{\mathbb E}
\def\F{\mathbf F}

\newcommand{\cA}{\mathcal A}

\newcommand{\cD}{\mathcal D}
\newcommand{\cE}{\mathcal E}

\newcommand{\cM}{\mathcal M}

\newcommand{\cP}{\mathcal P}

\newcommand{\cR}{\mathcal R}
\newcommand{\cS}{\mathcal S}

\newcommand{\cV}{\mathcal V}

\newcommand{\cX}{\mathcal X}

\title{DP-PCA:  Statistically Optimal and \\ Differentially Private PCA}

\author{%
  Xiyang Liu\footnotemark[0] \thanks{Paul  Allen School of Computer Science \& Engineering, 
  University of Washington, 
  \texttt{xiyangl@cs.washington.edu}} 
 \and Weihao Kong\footnotemark[0] \thanks{Google Research, \texttt{kweihao@gmail.com}} 
 \and Prateek Jain\footnotemark[0] \thanks{Google Research, \texttt{prajain@google.com}}
    \and
  Sewoong Oh\footnotemark[0] \thanks{Paul Allen School of Computer Science \& Engineering, 
  University of Washington, and  Google Research, 
  \texttt{sewoong@cs.washington.edu}} 
}

\date{}

\usepackage[numbers]{natbib}
\usepackage{graphicx}

\begin{document}

\maketitle

\begin{abstract}
We study the canonical statistical task of  computing the principal component from    $n$ i.i.d.~data in $d$ dimensions under $(\varepsilon,\delta)$-differential privacy. Although extensively studied in literature,  existing solutions fall short on two key aspects: ($i$) even for Gaussian data, existing private algorithms   require the number of samples  $n$ to scale super-linearly with $d$, i.e., $n=\Omega(d^{3/2})$, to obtain non-trivial results while non-private PCA  requires only $n=O(d)$, and ($ii$) existing techniques suffer from a non-vanishing  error even when the randomness in each data point is arbitrarily small. 
We propose DP-PCA, which is a single-pass algorithm that overcomes both limitations. It is based on a private minibatch gradient ascent method that relies on {\em private mean estimation}, which adds minimal noise required to ensure privacy by adapting to the variance of a given minibatch of gradients.
For sub-Gaussian data, we provide nearly optimal statistical error rates even for $n=\tilde O(d)$. 
Furthermore, we provide a lower bound  showing that sub-Gaussian style assumption is necessary in obtaining the optimal  error rate.  
\end{abstract}
%
\section{Introduction}
\label{sec:intro} 
Principal Component Analysis (PCA) is a fundamental statistical tool with multiple applications including  dimensionality reduction, data visualization, and noise reduction. Naturally, it is a key part of most standard data analysis and ML pipelines. 
However, when applied to data collected from numerous individuals, such as the U.S.~Census data, outcome of PCA might reveal highly sensitive personal information.  
We investigate the design of  privacy preserving PCA algorithms and the involved  privacy/utility tradeoffs, for computing the  first principal component, that should serve as the building block of more general rank-$k$ PCA. 




Differential privacy (DP) is a widely accepted mathematical notion of privacy introduced in \cite{dwork2006calibrating},   which is a standard in releasing  the U.S.~Census data \cite{census} and also deployed in  commercial systems \cite{apple,google1,google2}. 
A query to a database is said to be $(\varepsilon,\delta)$-differentialy private if a strong adversary who knows all other entries but one cannot infer that one entry from the query output, with high confidence. The  parameters  $\varepsilon$ and $\delta$ restricts the confidence as measured by the Type-I and II errors \cite{kairouz2015composition}. Smaller values of $\varepsilon\in[0,\infty)$ and $\delta\in[0,1]$ imply stronger privacy and  plausible deniability for the participants. 

For non-private PCA with $n$ i.i.d.~samples in $d$ dimensions, the popular Oja's algorithm (provided in Algorithm~\ref{alg:oja}) achieves the optimal error of $\sin(\hat v, v_1) = \tilde\Theta(\sqrt{d/n})$, where the error is measured by the sine function of the angle between the estimate, $\hat v$, and the principal component, $v_1$, \cite{jain2016streaming}. 
For differentially private PCA, there is a natural  fundamental question: {\em what is the extra cost we pay in the error rate for ensuring $(\varepsilon,\delta)$-DP?}

We introduce a novel approach we call DP-PCA (Algorithm~\ref{alg:adaptive_mean}) and show that it achieves an error bounded by $\sin(\hat v, v) = \tilde O( \sqrt{d/n} + d/(\varepsilon n) )$ for {\em sub-Gaussian-like} data defined in Assumption~\ref{asmp:A}, which is a broad class of light-tailed distributions that includes Gaussian data as a special case. The second term  characterizes the cost of privacy and this is tight; we prove a nearly matching information theoretic lower bound showing that no $(\varepsilon,\delta)$-DP algorithm can achieve a smaller error.  
This significantly improves upon a long line of existing private algorithms for PCA, e.g., \cite{PPCA,blum2005practical,hardt2013beyond,hardt2014noisy,dwork2014analyze}. These existing algorithms  are analyzed for fixed and non-stochastic data  and  achieve sub-optimal error rates of $O(\sqrt{d/n} + d^{3/2}/(\varepsilon n))$ even in  the stochastic  setting  we consider.  

A remaining question is whether the sub-Gaussian-like assumption, namely Assumption~\ref{asmp:A4}, is necessary or if it is an artifact of our analysis and our  algorithm. 
It turns out that such an assumption on the lightness of the tail is critical;  
we prove an algorithmic independent and information theoretic lower bound (Theorem~\ref{thm:lbht3}) to show that, 
without such an assumption, the cost of privacy  is lower bounded by $\Omega(\sqrt{d/(\varepsilon n)})$. 
This proves a separation of the error  depending on the lightness of the tail.

We start with the formal description of the stochastic setting in Section~\ref{sec:setting} and present Oja's algorithm for non-private PCA. 
Our first attempt in making this algorithm private in Section \ref{sec:privateoja} already achieves near-optimal error, if the data is strictly from a Gaussian distribution. However, there are two remaining challenges that we describe in detail in Section~\ref{sec:challenge}: $(i)$ the excessive number of iterations of Private Oja's Algorithm (Algorithm~\ref{alg:dpoja}) prevents using typical values of $\varepsilon$  used in practice, and $(ii)$ for general sub-Gaussian-like distributions, the error does not vanish even when the noise in the data (as measured by a certain fourth moment of a function of the data) vanishes. 
The first challenge is due to the analysis that  requires amplification by shuffling \cite{erlingsson2019amplification} that is restrictive. The second is due to its reliance on gradient norm clipping \cite{abadi2016deep} that does not adapt to the variance of the current gradients.
This drives the design of DP-PCA in Section \ref{sec:dppca} that critically relies on two techniques to overcome each challenge,  respectively. First, minibatch SGD (instead of single sample SGD) significantly reduces the number iterations, thus obviating the need for  amplification by shuffling. 
Next, private mean estimation (instead of gradient norm clipping and noise adding) adapts to the stochasticity of the problem and adds the minimal noise necessary to achieve privacy.
The main idea of this variance adaptive stochastic gradient update is explained in detail in Section \ref{sec:dppca-mean}, along with a sketch of a proof.

\medskip\noindent{\bf Notations.}
For a vector $x\in{\mathbb R}^d$, we use $\|x\|$ to denote the Euclidean norm. For a matrix $X\in {\mathbb R}^{d\times d}$, we use $\|X\|_2=\max_{\|v\|=1} \| X v \|_2$ to denote the spectral norm. We use ${\mathbf I}_{d}$ to denote $d\times d$ identity matrix.  For $n\in \mathbb{Z}^+$, let $[n] := \{1,2, \ldots, n\}$. Let $\mathbb{S}_2^{d-1}$ denote the unit $d$-sphere of $\ell_2$, i.e., $\mathbb{S}_2^{d-1}:=\{x\in \reals^d:\|x\|=1\}$. $\tilde{O}()$ hides logarithmic factors in $n$, $d$, and the failure probability $\zeta$.

\section{Problem formulation and background on DP} 
\label{sec:setting} 

Typical PCA assumes i.i.d.~data $\{x_i\in{\mathbb R}^d\}$ from a distribution and finds the first eigenvector of $\Sigma={\mathbb E}[(x_i-{\mathbb E}[x_i])(x_i-{\mathbb E}[x_i])^\top]\in{\mathbb R}^{d\times d}$. Our approach allows for a more general class of data $\{A_i\in{\mathbb R}^{d\times d}\}$ 
that recovers the standard case when $A_i=(x_i-{\mathbb  E}[x_i])(x_i-{\mathbb  E}[x_i])^\top$.

\begin{asmp}[$(\Sigma,\{\lambda_i\}_{i=1}^d,M,V,K,\kappa,a,\gamma^2)$-model]
\label{asmp:A}
	Let $A_1, A_2, \ldots, A_n\in \reals^{d\times d}$ be a sequence of (not necessarily symmetric) matrices sampled independently from the same distribution that satisfy the following with PSD matrices $\Sigma\in{\mathbb R}^{d\times d}$ and $H_u \in{\mathbb R}^{d\times d}$,  and positive scalar  parameters $M,V,K$, $\kappa$, $a$, and $\gamma^2$:
	\begin{enumerate}[noitemsep,nolistsep,leftmargin=*,label=\textbf{A.\arabic*.},ref=A.\arabic*]
		\item Let $\Sigma:=\E[A_i]$, for a symmetric positive semidefinite (PSD) matrix $\Sigma\in \reals^{d\times d}$, $\lambda_i$ denote the $i$-th largest eigenvalue of $\Sigma$, and $\kappa :=\lambda_1/(\lambda_1-\lambda_2)$,  \label{asmp:A1}
		\item $\|A_i-\Sigma\|_2\leq  \lambda_1 M$ almost surely, \label{asmp:A2}
		\item $\max\left\{\left\|\E\left[(A_i-\Sigma)(A_i-\Sigma)^\top\right]\right\|_2, \left\|\E\left[(A_i-\Sigma)^\top(A_i-\Sigma)\right]\right\|_2\right\}\leq \lambda_1^2 V$,
    \label{asmp:A3}
		\item  $\max_{\|u\|=1, \|v\|=1}\E\left[\exp\left(\left(\frac{|u^\top (A_i^\top-\Sigma) v|^2}{K^2\lambda_1^2 \|H_u\|_2}\right)^{1/(2a)}\right)\right]\leq 1$, where $H_u:=(1/\lambda_1^2) \E[(A_i-\Sigma)uu^\top(A_i -\Sigma)^\top]$.  We denote $\gamma^2 :=  \max_{\|u\|=1}\|H_u\|_2$. \label{asmp:A4}
	\end{enumerate}
\end{asmp}
The first three assumptions are required for PCA even if privacy is not needed. The last assumption provides a Gaussian-like tail bound  that determines how much noise we need to introduce in the algorithm for $(\varepsilon,\delta)$-DP. 
The following lemma is useful in the analyses. 

\begin{lemma} 
    \label{lem:tail}
    Under  \ref{asmp:A1} and  \ref{asmp:A4} in Assumption~\ref{asmp:A}, for any unit vector $u$, $v$, with probability $1-\zeta$,
		\begin{align}
			|u^\top (A_i^\top-\Sigma) v|^2 \;\leq\; K^2\lambda_1^2 \|H_u\|_2\log^{2a}(1/\zeta)\;.
		\end{align}
\end{lemma}


\subsection{Oja's algorithm}
In a non-private setting, the following  streaming algorithm  introduced in \cite{oja1982simplified} achieves optimal sample complexity as  analyzed in \cite{jain2016streaming}. 
It is a projected stochastic gradient ascent on the objective defined on the empirical covariance: $\max_{\|w\|= 1}  (1/n)\sum_{i=1}^n w^\top A_i w$.

\begin{algorithm2e}[H]    
   \caption{(Non-private) Oja's Algorithm} 
   \label{alg:oja} 
   	\DontPrintSemicolon 
	\SetKwProg{Fn}{}{:}{}
	{ 
	Choose $w_0$ uniformly at random from the unit sphere\\
	{\bf for} $t=1,2,\ldots, T$ {\bf do} 
		$w_{t}'\gets   w_{t-1}+\eta_tA_tw_{t-1}$ ,
		$w_{t} \gets w_{t}'/\|w_{t}'\|$ 
	\;
	Return $w_T$
	} 
\end{algorithm2e}

Central to our analysis is  the following  error bound on Oja's Algorithm from \cite{jain2016streaming}. 

\begin{theorem}[{\cite[Theorem~4.1]{jain2016streaming}}]
\label{thm:non-private-oja}

Under Assumptions~\ref{asmp:A1}-\ref{asmp:A3}, 
suppose the step size $\eta_t=\frac{\alpha}{(\lambda_1-\lambda_2)(\xi+t)}$ for some $\alpha>1/2$ and 
$
    \xi\; := \; 20\max\left(\kappa M\alpha, {\kappa^2\left(V+1\right)\alpha^2}/{\log(1+
    ({\zeta}/{100}))}\right)
$. 
If $T>\beta$ then there exists a constant $C>0$ such that Algorithm~\ref{alg:oja} outputs $w_T$ achieving w.p.~$1-\zeta$,
\begin{align}
	\sin^{2}\left(w_T, v_{1}\right)  
	\leq  \frac{C\log(1/\zeta)}{\zeta^2}\left(\, \frac{\alpha^2 \kappa^2 V}{(2\alpha-1)T} +  d\left(\frac{\xi}{T}\right)^{2\alpha}\, \right)\;.
	\label{eq:oja}
\end{align}
\end{theorem}

\subsection{Background on Differential Privacy}

Differential privacy (DP), introduced in \cite{dwork2006calibrating}, is a de facto mathematical measure for privacy leakage of a database accessed via queries. It ensures that even an adversary who knows all other entries cannot identify with a high confidence whether a person of interest participated in a database or not. 
\begin{definition}[Differential privacy \cite{dwork2006calibrating}]
    \label{def:dp}
    Given two multisets $S$ and $S'$, we say the pair $(S,S')$ is {\em neighboring} if 
 $ | S \setminus S'|+| S' \setminus S| \leq 1$. 
We say a stochastic query $q$ over a dataset $S$  satisfies 
$(\varepsilon,\delta)${\em-differential  privacy} for some $\varepsilon>0$ and $\delta\in(0,1)$ if ${\mathbb P}(q(S)\in A) \leq e^\varepsilon {\mathbb P}(q(S') \in A) + \delta$ for all neighboring $(S,S')$ and all subset $A$ of the range of $q$. 
\end{definition}

Small values of $\varepsilon$ and $\delta$ ensures that the adversary cannot identify any single data point with high confidence, thus providing plausible deniability. We provide useful DP lemmas in Appendix~\ref{sec:dp}. 
Within our stochastic gradient descent approach to PCA, we rely on the Gaussian mechanism to privatize each update. 
The {\em sensitivity} of a  query $q$ is defined as $\Delta_q:= \sup_{\text{neighboring }(S,S')} \|q(S) - q(S') \|$. 
    
\begin{lemma}[Gaussian mechanism \cite{dwork2014algorithmic}]
    \label{lem:gauss} 
    For a query $q$ with sensitivity $\Delta_q$, $\varepsilon\in(0,1)$, and $\delta\in(0,1)$, the Gaussian mechanism outputs $q(S)+{\cal N}(0,(\Delta_q (\sqrt{2\log(1.25/\delta)})/\varepsilon)^2{\mathbf I}_d)$
    and 
    achieves $(\varepsilon, \delta)$-DP. 
\end{lemma}

Another mechanism we frequently use is the private histogram learner of \cite{karwa2018finite}, whose analysis is provide in Appendix~\ref{sec:dp}, along with various composition theorems to provide end-to-end guarantees. 

\subsection{Comparisons with existing results in private PCA} 
\label{sec:comparisonPCA}

We briefly discuss the most closely related work and provide more  previous work in Appendix~\ref{sec:related}. 
Most existing results assume a fixed data under  a deterministic setting where each sample has a bounded norm,  $\|x_i\|\leq \beta$, and the goal is to find the top eigenvector of  $\hat\Sigma:=(1/n)\sum_{i=1}^n (x_i-\hat\mu)(x_i-\hat\mu)^\top  $ for the empirical mean $\hat\mu$. For the purpose of comparisons, consider  Gaussian $x_i\sim{\cal N}(0,\Sigma)$ with  $\|x_i\|\leq \beta=O(\sqrt{\lambda_1 d\log(n/\zeta)})$ for all $i\in[n]$ with probability $1-\zeta$.
The first line of approaches in  \cite{blum2005practical,PPCA,dwork2014analyze} is a 
Gaussian mechanism that outputs 
${\rm PCA}(\widehat \Sigma + Z)$, where $Z$ is a symmetric matrix with i.i.d.~Gaussian entries with a variance $((\beta^2/n\varepsilon)\sqrt{2\log(1.25/\delta)})^2$ to ensure $(\varepsilon,\delta)$-DP. The tightest result in  \cite[Theorem 7]{dwork2014analyze}  achieves 
\begin{eqnarray} 
\sin(\hat v,v_1) &=& \tilde O\Big(\kappa \Big( \sqrt{\frac{d}{n}} + \frac{ d^{3/2}\sqrt{\log(1/\delta)}}{ \varepsilon n} \Big) \,\Big) \;, \label{eq:base}
\end{eqnarray} 
with high probability,  
under a strong assumption that the spectral gap is very large: $\lambda_1-\lambda_2 = \omega(d^{3/2}\sqrt{\log(1/\delta)}/(\varepsilon n))$. 
In a typical scenario with $\lambda_1=O(1)$, this requires a large sample size of $n=\omega(d^{3/2}/\varepsilon)$.
Since this Gaussian mechanism does not exploit the statistical properties of i.i.d.~samples, the second term in this upper bound is 
larger by a factor of $d^{1/2}$
compared to the proposed DP-PCA (Corollary~\ref{cor:GaussUB}). 
 The  error rate of Eq.~\eqref{eq:base} is also achieved in   \cite{hardt2013beyond,hardt2014noisy} by adding Gaussian noise to the standard power method for computing the principal components. 
When the spectral gap, $\lambda_1-\lambda_2$, is smaller, it is possible to trade-off the dependence in $\kappa$ and the sampling ratio $d/n$, which we do not address in this work but is surveyed in Appendix~\ref{sec:related}.

\section{First attempt: making Oja's Algorithm private} 
\label{sec:privateoja} 

Following the standard recipe in training  with DP-SGD, e.g.,~\cite{abadi2016deep}, we introduce Private Oja's Algorithm in Algorithm~\ref{alg:dpoja}. At each gradient update, we first apply gradient norm clipping to limit the contribution of a single data point and next add an appropriately chosen Gaussian noise from Lemma~\ref{lem:gauss} to achieve $(\varepsilon,\delta)$-DP, end-to-end. 
The choice of clipping threshold $\beta$ ensures that, with high probability under our assumption, we do not clip any gradients. The choice of noise multiplier $\alpha$ ensures $(\varepsilon,\delta)$-DP.

\setlength{\textfloatsep}{5pt}
\begin{algorithm2e}[!t]    
   \caption{Private Oja's Algorithm} 
   \label{alg:dpoja} 
   	\DontPrintSemicolon 
	\KwIn{$S=\{A_i\in{\mathbb R}^{d\times d} \}_{i=1}^n$, privacy $(\varepsilon,\delta)$, learning rates $\{\eta_t\}^{n}_{t=1}$}
	\SetKwProg{Fn}{}{:}{}
	{ 
	Randomly permute $S$  \label{line:shuffle} and 
	choose $w_0$ uniformly at random from the unit sphere\\
	Set DP noise multiplier: $\alpha\gets  C' \log(n/\delta) / (\varepsilon\sqrt{n})$\\
	Set clipping threshold: $\beta\gets C\lambda_1\sqrt{d} (K \gamma\log^{a}(nd/\zeta)+1)$\\
	\For{t=1, 2, \ldots, n}{
		Sample $z_t\sim \cN(0, \mathbf{I}_d)$\\
		$w_{t}'\leftarrow  w_{t-1}+\eta_t \, {\rm clip}_\beta\left(A_t  w_{t-1}\right) +2\eta_t\beta\alpha z_t	$ where ${\rm clip}_\beta(x)=x\cdot\min\{1, \frac{\beta}{\|x\|_2}\}$\\
		$w_{t} \leftarrow w_{t}'/\|w_{t}'\|$ 
	}
	Return $w_n$
	} 
\end{algorithm2e}




One caveat in streaming algorithms is that we access data $n$ times, each with a private mechanism, but accessing only a single data point at a time. 
To prevent excessive privacy loss due to such a large number of data accesses, 
we apply a random shuffling in line~\ref{line:shuffle} Algorithm~\ref{alg:dpoja}, in order to benefit from a standard amplification by shuffling \cite{erlingsson2019amplification,feldman2022hiding}. This gives an amplified privacy guarantee that allows us to add a small noise proportional to  $ \alpha = O(\log(n/\delta)/(\varepsilon\sqrt{n}))$. Without the shuffle amplification, we will instead need  a larger noise scaling  as  $ \alpha = O(\log(n/\delta)/ \varepsilon )$, resulting in a suboptimal utility guarantee. 
However, this comes with a restriction that the amplification holds only for small values of  $\varepsilon=O(\sqrt{\log(n/\delta)/n})$. 
Our first contribution in the proposed DP-PCA (Algorithm~\ref{alg:adaptive_mean}) is to expand this range to $\varepsilon=O(1)$, which includes the practical regime of interest $\varepsilon\in[1/2,5]$. 

\begin{lemma}[Privacy]
If $\varepsilon=O(\sqrt{{\log(n/\delta)}/{n}} )$ and the noise multiplier is chosen to be  $\alpha=\Omega\left({\log(n/\delta)}/{(\varepsilon\sqrt{n})}\right)$, then Algorithm~\ref{alg:dpoja} is $(\varepsilon,\delta)$-DP.
\label{lem:dpoja_privacy}
\end{lemma}

Under Assumption~\ref{asmp:A}, we select gradient norm clipping threshold $\beta$ such that no gradient exceeds $\beta$. 

\begin{lemma}[Gradient clipping]
\label{lem:clipping_oja}
Let $\beta=C\lambda_1 \sqrt{d} (K\gamma\log^{a}(nd/\zeta)+1)$ for some  constant $C>0$. Then with probability $1-\zeta$, $\|A_t w_{t-1}\|\leq \beta$  for any fixed $w_{t-1}$ independent of $A_t$, for all $ t\in [n]$. 
\end{lemma}

We provide proofs of both lemmas and the next theorem in Appendix~\ref{sec:dpoja_proof}. 
When no clipping is applied, we can use the standard analysis of Oja's Algorithm from \cite{jain2016streaming}  to prove the following utility guarantee. 

\begin{theorem}[Utility]
    \label{thm:dpoja} 
Given $n$ i.i.d.~samples $\{A_i\in{\mathbb R}^{d\times d}\}_{i=1}^n$ satisfying Assumption~\ref{asmp:A} with parameters $(\Sigma,M,V,K,\kappa, a,\gamma^2)$, 
if  
\begin{eqnarray}
 n \;=\; \tilde O \Big( \, \kappa^2 + \kappa M  + \kappa^2 V + \frac{d\,\kappa\,(\gamma+1)\,\log(1/\delta)}{ \varepsilon}\,\Big) \;,
 \end{eqnarray}
 with a large enough constant, then there exists a positive universal constant $c_1$ and a choice of learning rate $\eta_t$ that depends on $(t,M$, $V$, $K$, $a$, $\lambda_1$, $\lambda_1-\lambda_2$, $n$,  $d$,  $\varepsilon$, $\delta)$ such that Algorithm~\ref{alg:dpoja} with a choice of $\zeta=0.01$ outputs $w_n$ that achieves  with probability $0.99$,
\begin{align}
	\sin^{2}\left(w_n, v_{1}\right)  
	\;=\; 
	\widetilde{O}\left(\kappa^2\Big(\frac{V}{n} + \frac{(\gamma+1)^2d^2\log^2(1/\delta)}{\varepsilon^2n^2} \,\Big)\,\right)\;,
	\label{eq:dpoja}
\end{align}
where $\widetilde{O}(\cdot)$ hides poly-logarithmic factors in $n$, $d$, $1/\varepsilon$, and $\log(1/\delta)$ and polynomial factors in $K$.
\end{theorem}

The first term in Eq.~\eqref{eq:dpoja} matches the non-private error rate for Oja's algorithm in Eq.~\eqref{eq:oja} with $\alpha=O(\log n)$ and $T=n$, and the second term is the price we pay for ensuring $(\varepsilon,\delta)$-DP. 

\begin{remark}
    For a canonical setting of a Gaussian data with $A_i=x_ix_i^\top$ and $x_i\sim{\cal N}(0,\Sigma)$, we have, for example from \cite[Lemma~1.12]{rigollet2015high}, that $M=O(d \log(n))$, $V=O(d )$, $K=4$, $a=1$, and $\gamma^2=O(1)$.    Theorem~\ref{thm:dpoja} implies the following error rate:
\begin{eqnarray}
    \label{eq:dpoja_Gauss}
     \sin^2\left(w_n, v_{1}\right) \;=\; \tilde O \Big( \kappa^2 \Big( \frac{d}{n} + \frac{d^2\log^2(1/\delta)}{\varepsilon^2 n^2} \Big) \Big) \;.
\end{eqnarray}
\label{rem:gauss}
\end{remark}
Comparing to a lower bound in Theorem~\ref{thm:lb}, this is already near optimal. However, for general distributions satisfying Assumption~\ref{asmp:A}, Algorithm~\ref{alg:dpoja} (in particular the second term in Eq.~\eqref{eq:dpoja}) can be significantly sub-optimal. 
We explain this second weakness of Private Oja's Algorithm in the following section (the first weakness is the restriction on $\varepsilon=O(\sqrt{\log(n/\delta)/n})$).

\vspace*{-5pt}
\section{Two remaining challenges} \vspace*{-3pt}
\label{sec:challenge}

We explain the two remaining challenges in Private Oja's Algorithm and propose techniques to overcomes these challenges that lead to the design of DP-PCA (Algorithm~\ref{alg:adaptive_mean}).

\medskip\noindent 
{\bf First challenge:  restricted range of $\varepsilon=O(\sqrt{\log(n/\delta)/n})$.} 
This
 is due to the large number, $n$, of iterations that necessitates the use of  
the amplification by shuffling in Theorem~\ref{thm:shuffling}. We  reduce the number of iterations with 
 minibatch SGD. For $T=O(\log^2 n)$ and $t=1,2,\ldots,T$, we repeat 
 \begin{eqnarray}
     \label{eq:minibatchsgd}
     w'_t \gets w_{t-1} +  \frac{\eta_t}{B}\sum_{i=1+B(t-1)}^{Bt-1}{\rm clip}_\beta(A_iw_{t-1}) + \frac{w\eta_t \beta \alpha }{B}z_t
     \;,\text{ and }\;\;
     w_t \gets w'_t/\|w_t'\| \;, 
 \end{eqnarray}
where $z_t\sim{\cal N}(0,{\bf I}_d)$ and the minibatch size is $B=\lfloor n/T\rfloor$.
Since the dataset is accessed only $T=O(\log^2 n)$ times, the end-to-end privacy is analyzed with the serial composition  (Lemma~\ref{lem:serial}) instead of the amplification by shuffling. This ensures $(\varepsilon,\delta)$-DP for any $\varepsilon=O(1)$, resolving the first challenge, and still achieves the utility guarantee of Eq.~\eqref{eq:dpoja}.

\medskip\noindent 
{\bf 
Second challenge: excessive noise for privacy. 
}  
This is best explained with an example. 
\begin{example}[Signal and noise separation]
\label{ex:toy}
Consider a setting with $A_i=x_ix_i^\top$ and 
$x_i \;=\; s_i + n_i$ where $s_i=v$ with probability half and $s_i=-v$ otherwise 
for a unit norm vector $v$ and $n_i\sim{\cal N}(0,\sigma^2 {\bf I})$. 
We want to find the principal component of $\Sigma={\mathbb E}[x_ix_i^\top]=vv^\top + \sigma^2{\bf I}$, which is $v$.
This construction decomposes the signal and the noise. For $A_i=vv^\top + s_in_i^\top + n_is_i^\top + n_in_i^\top$, the signal component is determined by $vv^\top$  that is deterministic due to the sign cancelling.  The noise component is $x_in_i^\top + n_is_i^\top + n_in_i^\top$  which is random. We can control the 
Signal-to-Noise Ratio (SNR), $1/\sigma^2$, by changing $\sigma^2$, and we are particularly interested in the regime where $\sigma^2$ is small. 
As we are interested in $\sigma^2<1$, this satisfies Assumption~\ref{asmp:A} with
$\lambda_1=1+\sigma^2$, $\lambda_2=\sigma^2$,  
$V=O(d \sigma^2)$, 
$K = O(1)$, $a=1$, and $\gamma^2 =\sigma^2$. Substituting this into Eq.~\eqref{eq:dpoja}, Private Oja's Algorithm achieves 
\begin{eqnarray}
    \label{eq:toy_oja} 
    \sin^2(w_n,v_1)\; = \; \tilde O\Big( \frac{\sigma^2 d}{n} + \frac{d^2\log(1/\delta)}{\varepsilon^2n^2}\Big) \;, 
\end{eqnarray}
where we are interested in $\sigma^2< 1$. 
\end{example}
This is problematic since the second term, due to the DP noise, does not vanish as the randomness $\sigma^2$ in the data decreases. 
We do not observe this for Gaussian data where  signal and noise scale proportionally as shown below. 
We reduce the noise we add for privacy, by switching from a simple norm clipping, that adds noise  proportional to the norm of the gradients, to private estimation, that only requires the noise to scale as the {\em range} of the gradients, i.e. the maximum distance between two gradients in the minibatch. The toy example above showcases that the range can be arbitrarily smaller than the maximum norm (Fig.~\ref{fig:grad}). 
We want to emphasize that although 
the idea of using private estimation within an optimization has been conceptually proposed in abstract settings, e.g., in \cite{kamath2021improved}, DP-PCA is the first setting where $(i)$ such separation between the norm and the range of the gradients holds  under any  statistical model, and hence $(ii)$ the long line of recent advances in private estimation provides significant gain over the simple DP-SGD \cite{abadi2016deep}.

\begin{figure}[!t]
\centering
\includegraphics[width=.45\textwidth]{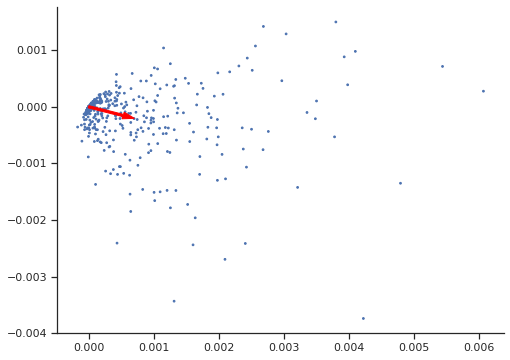}
\includegraphics[width=.45\textwidth]{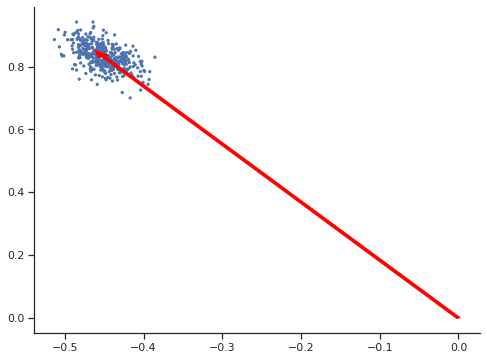}\vspace*{-5pt}
\caption{2-d  PCA under the Gaussian data from Remark~\ref{rem:gauss} (left) shows that the average gradient (red arrow) is smaller than  the range of the  minibatch of 400 gradients (blue dots). Under Example~\ref{ex:toy} (right), the range  can be made arbitrarily smaller than the average gradient by changing $\sigma^2$.} 
\label{fig:grad}
\end{figure}

\section{Differentially Private Principal Component Analysis (DP-PCA)}
\label{sec:dppca}

Combining the two ideas of minibatch SGD and private mean estimation, we propose DP-SGD. 
We use minibatch SGD of minibatch size $B=O(n/\log^2 n)$ to allow for larger range of $\varepsilon=O(1)$. We use Private Mean Estimation to add an appropriate level of noise chosen adaptively according to Private Eigenvalue Estimation. 
We describe details of both sub-routines in Section~\ref{sec:dppca-mean}.

\begin{algorithm2e}[!t]    
   \caption{Differentially Private Principal Component Analysis (DP-PCA)} 
   \label{alg:adaptive_mean} 
   	\DontPrintSemicolon 
	\KwIn{$S=\{A_i\}_{i=1}^n$, $(\varepsilon ,\delta )$, 
	batch size $B\in{\mathbb Z}_+$,
	learning rates $\{\eta_t\}_{t=1}^{\lfloor n/B\rfloor}$, 
	 probability $\zeta\in(0,1)$}
	\SetKwProg{Fn}{}{:}{}
	{ 
	Choose $w_0$ uniformly at random from the unit sphere\\
	\For{$t=1, 2, \ldots, T= \lfloor n/B\rfloor$}{ 
		Run Private Top Eigenvalue Estimation (Algorithm~\ref{alg:eigen}) with $(\varepsilon/2 , \delta/2 )$-DP and failure probability $\zeta/(2T)$ on $\{A_{B(t-1)+i} w_{t-1}\}_{i=1}^{\lfloor B/2\rfloor}$. Let the returned estimation be $\hat{\Lambda}_t>0$.\\	
	
		Run Private Mean Estimation (Algorithm~\ref{alg:dpmean}) with $\left(\varepsilon/2, \delta/2 \right)$-DP, failure probability $\zeta/(2T)$, and the estimated eigenvalue $2\hat{\Lambda}_t$ on $\left\{A_{B(t-1)+\lfloor B/2\rfloor+i} w_{t-1}\right\}_{i\in \lfloor B/2\rfloor}$. Let the returned mean gradient estimate  be $\hat{g}_t\in{\mathbb R}^d$.\\
		$w_{t}'\gets   w_{t-1}+\eta_t\hat{g}_t \;\;,\;\;\;\;	$ 
		$w_{t} \gets w_{t}'/\|w_{t}'\|$ 
	}
	Return $w_T$
	} 
\end{algorithm2e}

We show an upper bound on the error achieved by DP-PCA under an appropriate choice of the learning rate. We provide a complete proof in Appendix~\ref{sec:dppca_proof} that includes the explicit choice of the learning rate $\eta_t$ in Eq.~\eqref{eq:def_lr}, and a proof sketch is provided in Section~\ref{sec:sketch}.

 \begin{theorem}
    \label{thm:main} 
    For $\varepsilon\in(0,0.9)$, DP-PCA guarantees  $(\varepsilon,\delta)$-DP for all $S$, $B$, $\zeta$, and $\delta$. 
Given $n$ i.i.d.~samples $\{A_i\in{\mathbb R}^{d\times d}\}_{i=1}^n$ satisfying Assumption~\ref{asmp:A} with parameters $(\Sigma,M,V,K,\kappa,a,\gamma^2)$, 
if 
\begin{eqnarray}
 n \;=\; \tilde O \Big( \, e^{\kappa^2} + \frac{d^{1/2} (\log(1/\delta))^{3/2}}{\varepsilon}    
 + \kappa M + \kappa^2 V  + \frac{d\,\kappa\,\gamma\,(\log(1/\delta))^{1/2}}{\varepsilon}\,\Big) \;,
 \end{eqnarray}
 with a large enough constant and $\delta\leq 1/n$, then there exists a positive universal constant $c_1$ and a choice of learning rate $\eta_t$ that depends on $(t,M$, $V$, $K$, $a$, $\lambda_1$, $\lambda_1-\lambda_2$, $n$,  $d$,  $\varepsilon$, $\delta)$ such that $T=\lfloor n/B\rfloor $ steps of 
DP-PCA in Algorithm~\ref{alg:adaptive_mean} with  choices of $\zeta=0.01$ and $B=c_1n/(\log n)^2$, 
outputs $w_T$ such that with probability $0.99$,
\begin{align}
	\sin\left(w_T, v_{1}\right)  
	\;=\; 
	\widetilde{O}\left(\kappa \Big(\sqrt{\frac{V}{n}} + \frac{\gamma d \sqrt{\log(1/\delta)}}{\varepsilon n} \,\Big)\,\right)\;,
	\label{eq:main}
\end{align}
where $\widetilde{O}(\cdot)$ hides poly-logarithmic factors in $n$, $d$, $1/\varepsilon$, and $\log(1/\delta)$ and polynomial factors in $K$.
\end{theorem}

We further interpret this analysis and show that $(i)$ DP-PCA is nearly optimal when the data is from a  Gaussian distribution by comparing against a lower bound (Theorem~\ref{thm:lb}); and $(ii)$ DP-PCA significantly improves upon the private Oja's algorithm under Example~\ref{ex:toy}. 
We discuss the necessity of some of the assumptions at the end of this section, including how to agnostically find the appropriate learning rate scheduling. 

\medskip\noindent{\bf Near-optimality of DP-PCA under Gaussian distributions.} 
Consider the case of i.i.d.~samples $\{x_i\}_{i=1}^n$ from a Gaussian distribution from Remark~\ref{rem:gauss}.  

\begin{coro}[Upper bound; Gaussian distribution]
    \label{cor:GaussUB}
    Under the hypotheses of Theorem~\ref{thm:main} and $\{A_i=x_ix_i^\top\}_{i=1}^n$ with Gaussian random vectors  $x_i$'s, after $T=n/B$ steps, DP-PCA outputs $w_T$ that achieves, with  probability $0.99$,  
\begin{eqnarray}
    \label{eq:GaussUB} 
      \sin(w_T, v_1 ) 
\; = \;  \tilde O\left( \kappa \left (\sqrt{\frac{d}{n}}+\frac{d\sqrt{\log(1/\delta)}}{\varepsilon n}\right)
  \right) \; .
\end{eqnarray} 
\end{coro}
We prove a nearly matching lower bound, up to  factors of $\sqrt{\lambda_1/\lambda_2}$ and $\sqrt{\log(1/\delta)}$. One caveat is that the lower bound assumes  {\em pure}-DP with $\delta=0$. We do not yet have a lower bound technique for approximate DP that is tight, and all known approximate DP lower bounds have gaps to achievable upper bounds 
in its dependence in $\log(1/\delta)$, e.g.,    \cite{barber2014privacy,liu2021differential}. 
We provide a proof in Appendix~\ref{sec:lb_proof}. 

\begin{theorem}[Lower bound; Gaussian distribution]
    \label{thm:lb}
	 Let $\cM_\varepsilon$ be a class of $(\varepsilon,0)$-DP estimators that map  $n$ i.i.d. samples  to an estimate $\hat{v} \in{\mathbb R}^d$. A set of Gaussian distributions with $(\lambda_1,\lambda_2)$ as the first and second eigenvalues of the covariance matrix is denoted by $\cP_{(\lambda_1,\lambda_2)}$. There exists a universal constant $C>0$ such that
	\begin{align}
\inf_{\hat{v}\in \cM_{\varepsilon}}\sup_{P\in \cP_{(\lambda_1 , \lambda_2)} }\E_{S\sim P^n}\left[\sin(\hat{v}(S), v_1 )\right]
\; \geq \;  C\min\left( \kappa\left (\sqrt{\frac{d}{n}}+\frac{d}{\varepsilon n}\right)\sqrt{\frac{\lambda_2}{\lambda_1}}, 1\right)\;.
\label{eq:GaussLB} 
\end{align} 
\end{theorem}

\medskip \noindent
{\bf Comparisons with private 
Oja's algorithm.}  
We demonstrate that DP-PCA can significantly improve upon Private Oja's Algorithm with Example~\ref{ex:toy}, where DP-PCA achieves an error bound of 
$ \sin (w_T,v_1) = \tilde O \big( \sigma\sqrt{d/n} + \sigma d \sqrt{\log(1/\delta)}/(\varepsilon n ) \big)$. As the noise power $\sigma^2$ decreases DP-PCA achieves a vanishing error, whereas Private Oja's Algorithm has a non-vanishing error in  Eq.~\eqref{eq:toy_oja}. 
This follows from the fact that the second term in the error bound in Eq.~\eqref{eq:main} scales as $\gamma$, which can be made arbitrarily smaller than the second term in Eq.~\eqref{eq:dpoja} that scales as $(\gamma+1)$.  Further, the error bound for DP-PCA holds for any $\varepsilon=O(1)$, whereas Private Oja's Algorithm requires significantly smaller $\varepsilon=O(\sqrt{\log(n/\delta)/n})$.

\medskip\noindent 
{\bf Remarks on the assumptions of Theorem~\ref{thm:main}.} 
We have an exponential dependence of the sample complexity in the spectral gap, $n\geq  \exp(\kappa^2)$. This ensures we have a large enough  $T=\lfloor n/B \rfloor$ to reduce the non-dominant second term in Eq.~\eqref{eq:oja}, in balancing the learning rate $\eta_t$ and  $T$ (which is explicitly shown in Eqs.~\ref{eq:balance1} and \eqref{eq:balance2} in the Appendix). It is possible to get rid of this exponential dependence at the cost of an extra term of $\tilde O(\kappa^4\gamma^2d^2\log(1/\delta)/(\varepsilon n )^2)$ in the error rate in Eq.~\eqref{eq:main}, 
 by selecting a slightly larger $T = c \kappa^2 \log^2 n$. 
A Gaussian-like tail bound in Assumption~\ref{asmp:A4} is necessary to get the desired upper bound scaling as $\tilde O( d\sqrt{\log (1/\delta)}/(\varepsilon n))$ in Eq.~\ref{eq:GaussUB}, for example. 
The next lower bound shows that without such assumptions on the tail, the error due to privacy scales as $\Omega(\sqrt{d\wedge \log(1/\delta) / (\varepsilon n)}) $. We believe that the dependence in $\delta$ is loose, and it might be possible to get a tighter lower bound using  \cite{kamath2022new}. 
We provide a proof and other lower bounds in Appendix~\ref{sec:lb}.  


\begin{theorem}[Lower bound without Assumption~\ref{asmp:A4}]
    \label{thm:lbht3}
	 Let $\cM_\varepsilon$ be a class of $(\varepsilon,\delta)$-DP estimators that map  $n$ i.i.d.~samples  to an estimate $\hat{v} \in{\mathbb R}^d$.  A set of distributions satisfying Assumptions~\ref{asmp:A1}--\ref{asmp:A3} with 
	 $M=\tilde{O}(d+\sqrt{n\varepsilon/d} )$,  $V=O(d)$ and $\gamma=O(1)$ is denoted by $\tilde \cP$. 
	 For $d\geq 2 $, there exists a universal constant $C>0$ such that
	\begin{align}
\inf_{\hat{v}\in \cM_{\varepsilon}} 
\sup_{P\in \tilde \cP}\E_{S\sim P^n}\left[\sin(\hat{v}(S), v_1 )\right] 
\; \geq \;  C\kappa\min\left( \sqrt{\frac{ d \wedge \log \left(\left(1-e^{-\varepsilon}\right) / \delta\right)}{\varepsilon n}}, 1\right)\;.
\label{eq:htLB3} 
\end{align} 
\end{theorem}

Currently, DP-PCA requires choices of the learning rates, $\eta_t$, that depend on possibly unknown quantities. Since we can privately evaluate the quality of our solution, one can instead run multiple instances of DP-PCA with varying  $\eta_t = c_1/(c_2 + t)$ and find the best choice of $c_1>0$ and $c_2>0$. Let $w_T(c_1,c_2)$ denote the resulting solution for one instance of $\{\eta_t = c_1/(c_2 + t)\}_{t=1}^T$. We first set a target error $\zeta$. For each round $i=1,\ldots$, we will run algorithm for $(c_1, c_2) = [2^{i-1},2^{-i+1}]\times [2^{-i+1}, 2^{-i+2}\ldots, 2^{i-1}]$ and $(c_1, c_2) = [2^{-i+1}, 2^{-i+2}\ldots, 2^{i-1}]\times [2^{i-1},2^{-i+1}]$, and compute each $\sin(w_T(c_1, c_2), v_1)$ privately, each with privacy budget $\eps_i = \frac{\eps}{2^{i+1}(2i-1)}, \delta_i = \frac{\delta}{2^{i+1}(2i-1)}$. We terminate the algorithm once there there is a $w_T(c_1, c_2)$ satisfies $\sin(w_T(c_1, c_2), v_1)\le \zeta$. It is clear that this search meta-algorithm terminate in logarithmic round, and the total sample complexity only blows up by a poly-log factor.







\section{Private mean estimation for the minibatch stochastic gradients}
\label{sec:dppca-mean}

DP-PCA critically relies on private mean estimation to reduce  variance of the noise required to achieve $(\varepsilon,\delta)$-DP. 
We follow a common recipe from  \cite{karwa2018finite,kamath2019privately,kamath2020private,biswas2020coinpress, covington2021unbiased}.
First, we privately find an approximate range of the gradients in the minibatch (Alg.~\ref{alg:eigen}). Next, we apply the Gaussian mechanism to the truncated gradients where the truncation is tailored to the estimated range (Alg.~\ref{alg:dpmean}). 

\medskip\noindent{\bf Step 1: estimating the range.} We need to find an approximate range of the minibatch of gradients in order to adaptively truncate the gradients and bound the sensitivity. Inspired by a private preconditioning mechanism designed for mean estimation with unknown covariance from \cite{kamath2021private}, we propose to use privately estimated top eigenvalue of the covariance matrix of the gradients. 
For details on the version of the histogram learner we use in Alg.~\ref{alg:eigen} in Appendix~\ref{sec:eigen_proof}, we refer to \cite[Lemma D.1]{liu2021robust}. 
Unlike the private preconditioning of \cite{kamath2021private} that estimates all eigenvalues and requires $n=\widetilde{O}(d^{3/2}\log(1/\delta)/\varepsilon)$ samples, we only require the top eigenvalue and hence the next theorem shows that we only need $n=\widetilde{O}(d \log(1/\delta)/\varepsilon)$.

\begin{theorem}
    \label{thm:eigen}
	Algorithm~\ref{alg:eigen} is $(\varepsilon, \delta)$-DP. Let $g_i=A_i u$ for some fixed  vector $u$, where $A_i$ satisfies \ref{asmp:A1} and \ref{asmp:A4}  in Assumption~\ref{asmp:A} such that the mean is $\E[g_i] = \Sigma u$ and the covariance is  $\E[(g_i-\Sigma u)(g_i-\Sigma u)^\top] = \lambda_1^2 H_u $. With a large enough sample size scaling as 
	\begin{align}
		B &= O\left(\frac{K^2 \,d
		\,\log^{1+2a}(nd\log(1/(\delta\zeta))/\varepsilon)\log(1 /(\zeta\delta))}{\varepsilon}\right)
		=\tilde{O}\left(\frac{K^2 \,d\, \log(1 /\delta)}{\varepsilon}\right)\;,
	\end{align}
	Algorithm~\ref{alg:eigen}  outputs $\hat{\Lambda}$ achieving  $\hat{\Lambda}\in \left[(1/\sqrt{2})\lambda_1^2\|H_u\|_2, \sqrt{2} \lambda_1^2 \|H_u\|_2 \right]$  with probability $1-\zeta$, where the pair $(K>0,a>0 )$ parametrizes the tail of the distribution in \ref{asmp:A4} and  $\tilde{O}(\cdot)$ hides logarithmic factors in $B,d,1/\zeta, \log(1/\delta)$, and  $\varepsilon$.
\end{theorem}

We provide a proof in Appendix~\ref{sec:eigen_proof}.
There are other ways to privately estimate the range. 
Some approaches require known bounds such as $ \sigma_{\rm min}^2 \leq \lambda_1^2 (  H_u)_{ii} \leq \sigma_{\rm max}^2 $ for all $i\in[d]$ \cite{karwa2018finite}, and other agnostic approaches  are more involved such as instance optimal universal estimators of \cite{dong2021universal}. 

\medskip\noindent{\bf Step 2: Gaussian mechanism for mean estimation.} 
Once we have a good estimate of the top eigenvalue from the previous section, we use it to select the bin size of the private histogram and compute the  truncated empirical mean. Since truncated empirical mean has a bounded sensitivity, we can use Gaussian mechanism to achieve DP. The algorithm is now standard in DP mean estimation, e.g., \cite{karwa2018finite,kamath2019privately}. However, the analysis is slightly different since our  assumptions on $g_i$'s are different. 
For completeness, we provide the Algorithm~\ref{alg:dpmean} in Appendix~\ref{sec:dpmean_proof}. 

The next lemma shows that the Private Mean Estimation is $(\varepsilon,\delta)$-DP, and with high probability clipping does not apply to any of the gradients.  The returned private mean, therefore, is distributed as a spherical Gaussian centered at the  empirical mean of the gradients. This result requires that we have a good estimate of the top eigenvalue from Alg.~\ref{alg:eigen} such that $\hat{\Lambda}\simeq \lambda_1^2 \|H_u\|_2$. This analysis implies that we get an unbiased estimate of the  gradient mean (which is critical in the analysis) with noise scaling as $\tilde O(\lambda_1 \gamma\sqrt{d\log(1/\delta)} /(\varepsilon B))$, where $\gamma^2=\max_{u:\|u\|=1} \|H_u\|_2$ (which is critical in getting the tight sample complexity in the second term of the final utility guarantee in Eq.~\eqref{eq:main}). We provide a proof in Appendix~\ref{sec:dpmean_proof}. 

\begin{lemma}
	For $\varepsilon\in(0,0.9)$ and any $\delta\in(0,1)$, Algorithm~\ref{alg:dpmean} is $(\varepsilon, \delta)$-DP. Let $g_i=A_i u$ for some fixed  vector $u$, where $A_i$ satisfies \ref{asmp:A1} and \ref{asmp:A4}  in Assumption~\ref{asmp:A} such that the mean is $\E[g_i] = \Sigma u$ and the covariance is  $\E[(g_i-\Sigma u)(g_i-\Sigma u)^\top] =\lambda_1^2 H_u $. If $\hat{\Lambda}\in [\lambda_1^2\|H_u\|_2/\sqrt{2}, \sqrt{2}\lambda_1^2 \|H_u\|_2]$, $\delta\leq 1/B$, and  $B=\Omega((\sqrt{d\log(1/\delta)}/\varepsilon)\log(d/(\zeta\delta)))$ then,  with probability $1-\zeta$, we have $g_i\in \bar{g}+\left[-3K\sqrt{\hat{\Lambda}}\log^a(Bd/\zeta), 3K\sqrt{\hat{\Lambda}}\log^a(Bd/\zeta)\right]^d$ for all $i\in [B].$
	\label{lem:dpmean}
\end{lemma}

\subsection{Proof sketch of Theorem~\ref{thm:main}} 
\label{sec:sketch}

We choose $B=\Theta(n/\log^2 n)$ such that we access the dataset only $T=\Theta(\log^2 n)$ times. Hence we do not need to rely on amplification by shuffling. 
To add Gaussian noise that scales as the standard deviation of the gradients in each minibatch (as opposed to potentially excessively large mean of the gradients), DP-PCA adopts techniques from recent advances in private mean estimation. Namely, we first get a private and accurate estimate of the  range from Theorem~\ref{thm:eigen}. Using this estimate, $\hat{\Lambda}$, Private Mean Estimation returns an unbiased estimate of the empirical mean of the gradients, as long as no truncation has been applied as ensured by Lemma~\ref{lem:dpmean}. This gives 
\begin{align}
		w_{t}'&\gets  w_{t-1}+\eta_t \left(\frac{1}{B}\sum_{i=1}^B A_{B(t-1)+i}w_{t-1}  +\beta_t z_t\right) \;,
\end{align}
for $z_t\sim{\cal N}(0,{\bf I})$ and  $\beta_t=\frac{8K\sqrt{2\hat{\Lambda}_t}\log^a(Bd/\zeta)\sqrt{2d\log(2.5/\delta)}}{\varepsilon B}$. Using rotation invariance of spherical Gaussian random vectors and the fact that $\|w_{t-1}\|=1$, we can reformulate it as 
\begin{align}
		w_{t}'&\gets  w_{t-1}+\eta_t \underbrace{\left(\frac{1}{B}\sum_{i=1}^B A_{B(t-1)+i} +\beta_t G_t\right)}_{\tilde A_t}w_{t-1} \;.
\end{align}
This process can be analyzed with Theorem~\ref{thm:non-private-oja} with $\tilde A_t$ substituting $A_t$.

\section{Discussion}
\label{sec:discussion}

Under the canonical task of computing the principal component from i.i.d.~samples, we show the first result achieving a near-optimal error rate. This critically relies on two ideas: minibatch SGD and private mean estimation. In particular, private mean estimation plays a critical role in the case when the range of the gradients is significantly smaller than the norm; we achieve an optimal error rate that cannot be achieved with the standard recipe of gradient clipping or even with a more sophisticated adaptive clipping \cite{andrew2021differentially}.

Assumption~\ref{asmp:A4} can be relaxed to heavy-tail bounds with bounded $k$-th moment on $A_i$, in which case we expect the second term in Eq.~\eqref{eq:main} to scale as 
$O(d(\sqrt{\log(1/\delta)}/\varepsilon n)^{1-1/k})$, drawing analogy from a similar trend in a computationally inefficient DP-PCA without spectral gap  \cite[Corollary 6.10]{liu2021differential}. 
 When a fraction of data is corrupted, recent advances in \cite{xu2010principal,kong2020robust,jambulapati2020robust} provide optimal algorithms for PCA. However, existing approach of \cite{liu2021differential} for robust and private PCA is computationally intractable. 
Borrowing ideas from robust and private mean estimation in \cite{liu2021robust}, one can design an efficient algorithm, but at the cost of sub-optimal sample complexity. It is an interesting direction to design an optimal and robust version of DP-PCA.
Our lower bounds are loose in its dependence in $\log(1/\delta)$. Recently, a  promising lower bound technique has been introduced in \cite{kamath2022new} that might close  this gap. 

\section*{Acknowledgement} 
This work is supported in part by Google faculty research award and NSF grants CNS-2002664, IIS-1929955, DMS-2134012, CCF-2019844 as a part of NSF Institute for Foundations of Machine Learning (IFML), CNS-2112471 as a part of NSF AI Institute for Future Edge Networks and Distributed Intelligence (AI-EDGE).

\bibliographystyle{plain}
\bibliography{ref}

\newpage 
\appendix
\section*{Appendix}
\section{Related work}
\label{sec:related}

Our work builds upon a series of advances in 
private SGD \cite{kamath2021improved, bassily2014private, bassily2019private, feldman2020private, kulkarni2021private, wang2020differentially, hu2021high}
to make  advance in understanding the tradeoff of privacy and sample complexity for PCA. Such tradeoffs have been studied extensively in canonical statistical estimation problems of mean (and covariance) estimation and linear regression. 

\medskip\noindent
\textbf{Private mean estimation.} 
As one of the most fundamental problem in private data analysis, mean estimation is initially studied under the bounded support assumptions, and the optimal error rate is now well understood. More recently,~\cite{barber2014privacy} considered the private mean estimation problem for $k$-th moment bounded distributions where the support of the data is \textit{unbounded} and provided minimax error bound in various settings. \cite{karwa2018finite} studied private mean estimation from Gaussian sample, and obtained an optimal error rate. There has been a lot of recent interests on private mean estimation under various assumptions, including mean and covariance joint estimation ~\cite{kamath2019privately, biswas2020coinpress}, heavy-tailed mean estimation~\cite{kamath2020private}, mean estimation for general distributions~\cite{feldman2018calibrating, tzamos2020optimal}, distribution adaptive mean estimation~\cite{bun2019average}, estimation for unbounded distribution parameters~\cite{kamath2021private}, mean estimation under pure differential privacy~\cite{hopkins2021efficient}, local differential privacy~\cite{duchi2019lower, duchi2018minimax, gaboardi2019locally, joseph2019locally,li2022robustness}, user-level differential privacy~\cite{esfandiari2021tight,esfandiari2021tight}, Mahalanobis distance\cite{brown2021covariance,liu2021differential} and robust and differentially private mean estimation~\cite{liu2021robust, kothari2021private, liu2021differential}.

\medskip\noindent
\textbf{Private linear regression} 
The goal of private linear regression is to learn a linear predictor of response variable $y$ from a set of examples $\{x_i, y_i\}_{i=1}^n$ while guarantee the privacy of the examples. Again, the work on private linear regression can be divided into two categories: deterministic and randomized. In the deterministic setting where the data is deterministically given without any probabilistic assumptions, significant advances in DP linear regression has been made 
\cite{vu2009differential,kifer2012private,mir2013differential,dimitrakakis2014robust,bassily2014private,wang2015privacy,foulds2016theory, minami2016differential, wang2018revisiting, sheffet2019old}. In the randomized settings where each example $\{\x_i, y_i\}$ is drawn i.i.d. from a distribution, \cite{dwork2009differential} proposes an exponential time algorithm that achieves asymptotic consistency. \cite{cai2019cost} provides an efficient and minimax optimal algorithm under sub-Gaussian design and nearly identity covariance assumptions. Very recently, \cite{liu2021differential} for the first time gives an exponential time algorithm that achieves minimax risk for general covariance matrix under sub-Gaussian and hypercontractive assumptions.

\medskip\noindent{\bf Private PCA without spectral gap.}
There is a long line of work in 
Private PCA 
\cite{hardt2012beating, hardt2013beyond, hardt2014noisy, blum2005practical, PPCA, kapralov2013differentially, dwork2014analyze, balcan2016improved}. We explain the closely related ones in Section~\ref{sec:comparisonPCA}, providing interpretation when the covariance matrix has a spectral gap.   

When there is no spectral gap, one can still learn a principal component. However, since the principal component is not unique, the error is typically measured in how much of the variance is captured in the estimated direction: 
$1-\hat v^\top\Sigma \hat v/\|\Sigma\|$. 
    \cite{PPCA} introduces  an exponential mechanism (from \cite{mcsherry2007mechanism}) which samples an estimate from a distribution $f_{\widehat \Sigma}(\hat v) =(1/C)\exp\{ ((\varepsilon n)/c^2)\hat v^\top \widehat \Sigma \hat v\}$, where $C$ is a normalization constant to ensure that the pdf integrates to one. This achieves a stronger pure DP, i.e., $(\varepsilon,0)$-DP, but is computationally expensive;  \cite{PPCA} does not provide a tractable implementation and  \cite{kapralov2013differentially} provides a polynomial time implementation with time complexity at least cubic in  $d$. 
This achieves an error rate of $1-\hat v^\top\Sigma \hat v/\|\Sigma\| =\tilde O(d^2 /(\varepsilon n)) $ in  \cite[Theorem 7]{PPCA} when samples are from Gaussian ${\cal N}(0,\Sigma)$, which,  when there is a spectral gap, translates into 
\begin{eqnarray}
    \sin(\hat v, v_1)^2 &=& \tilde O\Big(\frac{\kappa  d^2}{ \varepsilon n }\Big) \;,    
\end{eqnarray}
with high probability. 
Closest to our setting is the analyses in \cite[Corollary 6.5]{liu2021differential} that proposed an exponential mechanism that achieves 
$1-\hat v^\top\Sigma \hat v/\|\Sigma\| =\tilde O(\sqrt{d/n} + (d+\log(1/\delta))/(\varepsilon n)) $ with high probability under 
$(\varepsilon,\delta)$-DP and  Gaussian samples, but this algorithm is  computationally intractable. This is shown to be tight when there is no spectral gap. When there is a spectral gap, this translates into 
\begin{eqnarray}
    \sin(\hat v, v_1)^2 &=& \tilde O\Big( \kappa 
    \Big(\sqrt{\frac{d}{n}} + \frac{d+\log(1/\delta)}{\varepsilon n}\Big) \Big) \;.     
\end{eqnarray} 
As these algorithms and the corresponding analyses are tailored for gap-free cases, they have better dependence on  $\kappa$  and worse dependence on $d/n$ and $d/\varepsilon n$, 
compared to the proposed DP-PCA and its error rate in Corollary~\ref{cor:GaussUB}. 

\section{Preliminary on differential privacy}
\label{sec:dp}


\begin{lemma}[Stability-based histogram {\cite[Lemma~2.3]{karwa2018finite}}]\label{lem:hist-KV17} For every $K\in \mathbb{N}\cup \infty$, domain $\Omega$, for every collection of disjoint bins $B_1,\ldots, B_K$ defined on $\Omega$, $n\in \mathbb{N}$, $\eps\geq 0,\delta\in(0,1/n)$, $\beta>0$ and $\alpha\in (0,1)$ there exists an $(\eps,\delta)$-differentially private algorithm $M:\Omega^n\to \mathbb{R}^K$ such that for any set of data $X_1,\ldots,X_n\in \Omega^n$
\begin{enumerate}
\item $\hat{p}_k = \frac{1}{n}\sum_{X_i\in B_k}1$
\item $(\tilde{p}_1,\ldots,\tilde{p}_K)\gets M(X_1,\ldots,X_n),$ and
\item
$$
n\ge \min\left\{\frac{8}{\eps\beta}\log(2K/\alpha),\frac{8}{\eps\beta}\log(4/\alpha\delta)\right\} 
$$
\end{enumerate}
then,
$$
\mathbb{P}(|\tilde{p}_k-\hat{p}_k|\le\beta)\ge 1-\alpha
$$
\end{lemma}

Since we focus on one-pass algorithms where a data point is only accessed once, we need a basic parallel composition of DP. 

\begin{lemma}[Parallel composition \cite{mcsherry2009privacy}]
    \label{lem:parallel} Consider a sequence of interactive queries 
    $\{q_k\}_{k=1}^K$ each operating on a subset $S_k$ of the database and each satisfying  $(\varepsilon, \delta)$-DP.
    If  $S_k$'s are disjoint then the composition $(q_1(S_1), q_2(S_2), \ldots, q_K(S_K))$ is
    $(\varepsilon,\delta)$-DP.
\end{lemma}

We also utilize the following serial composition theorem. 

\begin{lemma}[Serial composition \cite{dwork2014algorithmic}]
    \label{lem:serial} 
    If a database is accessed with an $(\varepsilon_1,\delta_1)$-DP mechanism and then with an $(\varepsilon_2,\delta_2)$-DP mechanism, then the end-to-end privacy guarantee is $(\varepsilon_1+\varepsilon_2,\delta_1+\delta_2)$-DP.
\end{lemma}
When we apply private histogram learner to each coordinate, we require more advanced composition theorem from \cite{kairouz2015composition}. 
\begin{lemma}[Advanced composition \cite{kairouz2015composition}]
    \label{lem:composition}
    For $\varepsilon\leq0.9$, 
    an end-to-end guarantee of $(\varepsilon,\delta)$-differential privacy is satisfied if a database  is accessed $k$ times, each with  a $(\varepsilon/(2\sqrt{2k\log(2/\delta)}),\delta/(2k))$-differential private mechanism. 
\end{lemma}

\section{Lower bounds}
\label{sec:lb}

When privacy is not required, we know from 
Theorem~\ref{thm:non-private-oja} 
that under Assumptions \ref{asmp:A1}-\ref{asmp:A3}, we can achieve an error rate of $\tilde O(\kappa \sqrt{V/n})$. In the regime of $V=O(d)$ and $\kappa=O(1)$, $n=O(d)$  samples are enough to achieve an arbitrarily small error. The next lower bounds shows that we need $n=O(d^2)$ samples  when $(\varepsilon=O(1),0)$-DP is required, showing that private PCA is significantly more challenging than a non-private PCA, when assuming only the support and moment bounds of assumptions~\ref{asmp:A1}-\ref{asmp:A3}. 
We provide a proof in Appendix~\ref{sec:lbht_proof}.

\begin{theorem}[Lower bound without Assumption~\ref{asmp:A4}]
    \label{thm:lbht}
	 Let $\cM_\varepsilon$ be a class of $(\varepsilon,0)$-DP estimators that map  $n$ i.i.d.~samples  to an estimate $\hat{v} \in{\mathbb R}^d$.  A set of distributions satisfying Assumptions~\ref{asmp:A1}--\ref{asmp:A3} with 
	 $M=O(d  \log n )$ and  $V=O(d)$ is denoted by $\tilde \cP_{(\lambda_1,\lambda_2)}$. 
	 There exists a universal constant $C>0$ such that
	\begin{align}
\inf_{\hat{v}\in \cM_{\varepsilon}} 
\sup_{P\in \tilde \cP_{(\lambda_1,\lambda_2)}}\E_{S\sim P^n}\left[\sin(\hat{v}(S), v_1 )\right] 
\; \geq \;  C\min\left( \frac{\kappa d^2}{\varepsilon n} \sqrt{\frac{\lambda_2}{\lambda_1}}, \sqrt{\frac{\lambda_2}{\lambda_1}}\right)\;.
\label{eq:htLB} 
\end{align} 
\end{theorem}

We next provide the proofs of all the lower bounds.


\subsection{Proof of Theorem~\ref{thm:lb} on the lower bound for Gaussian case} 
\label{sec:lb_proof}

Our proof is based on following differentially private Fano's method \cite[Corollary~4]{acharya2021differentially}.

\begin{theorem}[DP Fano's method {\cite[Corollary~4]{acharya2021differentially}}]
\label{thm:dp_fano}
	 Let $\cP$ denote family of distributions of interest and $\theta:\cP\rightarrow\Theta$ denote the population parameter. Our goal is to estimate   $\theta$ from i.i.d. samples $x_1, x_2, \ldots, x_n\sim P\in \cP$. Let $\hat{\theta}_\varepsilon$ be an $(\varepsilon, 0)$-DP estimator. Let  $\rho:\Theta\times\Theta \rightarrow \reals^+$ be a
pseudo-metric on parameter space $\Theta$. 
	 Let $\cV$ be an index set with finite cardinality. Define $\cP_\cV=\{P_v, v\in \cV\} \subset \cP$ be an indexed family of probability measures on measurable set $(\cX, \cA)$. If for any $v\neq v'\in \cV$, 
	\begin{enumerate}
		\item $\rho(\theta(P_v), \theta(P_{v'}))\geq \tau$,
		\item $D_{\rm KL}\left(P_v, P_{v'}\right)\leq \beta$,
		\item $D_{\rm TV}\left(P_v, P_{v'}\right) \leq \phi $,
	\end{enumerate}
	
	then 	\begin{align*}
		\inf_{\hat{\theta}_\varepsilon}\max_{P\in \cP}\E_{S\sim P^n}\left[\rho(\hat{\theta}_\varepsilon(S), \theta(P))\right]\geq \max\left( \frac{\tau}{2}\left(1-\frac{n\beta+\log(2)}{\log(|\cV|)}\right),0.4 \tau \min\left(1, \frac{\log(|\cV|)}{n \phi \varepsilon}\right)\right)\;.
	\end{align*} 
\end{theorem}

For our problem, we are interested in Gaussian $\cP_\Sigma$ and metric $\rho(u,v)=\sin(u,v)$.  Using Theorem~\ref{thm:dp_fano}, it suffices to construct such indexed set $\cV$ and the indexed distribution family $\cP_\cV$. We use the same construction as in \cite[Theorem~2.1]{vu2012minimax} introduced to prove a lower bound for the (non-private) sparse PCA problem. The construction is given by the following lemma.

\begin{lemma}[{\cite[Lemma~3.1.2]{vu2012minimax}}]
\label{lem:packing_alpha}
	Let $d>5$. For $\alpha\in (0,1]$, there exists $\cV_\alpha \subset \mathbb{S}_2^{d-1}$ and an absolute constant $c_1>0$ such that for every $v\neq v'\in \cV_\alpha$, 
	$
		\alpha/\sqrt{2} \leq \|v-v'\|_2\leq \sqrt{2}\alpha
	$
	and
	$
		\log(|\cV_\alpha|)\geq c_1d
	$.
\end{lemma}

Fix $\alpha\in (0, 1]$. For each $v\in \cV_\alpha$, we define $\Sigma_v= (\lambda_1-\lambda_2)vv^\top+\lambda_2\mathbf{I}_d$ and $P_v = \cN(0, \Sigma_v)$. It is easy to see that $\Sigma_v$ has eigenvalues $\lambda_1>\lambda_2=\cdots=\lambda_n$. The top eigenvector of $\Sigma_v$ is $v$. Using Lemma~\ref{lem:sin2distance}, we know for any $v\neq v'\in \cV$,
\begin{align}
	\frac{\alpha}{\sqrt{2}}\leq \rho(v, v')=\sqrt{1-\ip{v}{v'}^2}\leq \alpha\;.
\end{align}

Using \cite[Lemma~3.1.3]{vu2012minimax}, we know
\begin{align}
	D_{\rm KL}\left(P_v, P_{v'}\right)= \frac{(\lambda_1-\lambda_2)^2}{\lambda_1\lambda_2}(1-\ip{v}{v'}^2)\leq \frac{(\lambda_1-\lambda_2)^2\alpha^2}{\lambda_1\lambda_2}\;.
\end{align}
Using Pinsker's inequality, we have
\begin{align}
	D_{\rm TV}\left(P_v, P_{v'}\right)\leq \sqrt{\frac{D_{\rm KL}\left(P_v, P_{v'}\right)}{2}}\leq \alpha\sqrt{\frac{(\lambda_1-\lambda_2)^2}{2\lambda_1\lambda_2}}\;.
\end{align}

Now we set \begin{align}
	\alpha:= \min\left(1, \sqrt{\frac{dc_1\lambda_1\lambda_2}{2n(\lambda_1-\lambda_2)^2}}, \frac{c_1 d}{n\varepsilon}\sqrt{\frac{2\lambda_1\lambda_2}{(\lambda_1-\lambda_2)^2}}\right)
\end{align}

It follows from Theorem~\ref{thm:dp_fano} and $d>8$ that there exists a constant $C$ such that
\begin{align}
	\inf_{\hat{v}}\sup_{P\in \cP_\Sigma}\E_{S\sim P^n}\left[\sin(\hat{v}(S), v_1(\Sigma))\right]
\geq  
C\min\left(\left (\sqrt{\frac{d}{n}}+\frac{d}{\varepsilon n}\right)\sqrt{\frac{\lambda_1\lambda_2}{(\lambda_1-\lambda_2)^2}}, 1\right)\;.
\end{align}

\subsection{Proof of Theorem~\ref{thm:lbht3} }
\label{sec:lbht_proof3}

We first construct an indexed set $\cV$ and indexed distribution family $\cP_\cV$ such that $x_ix_i^\top$ satisfies \ref{asmp:A1}, \ref{asmp:A2} and \ref{asmp:A3} in Assumption~\ref{asmp:A}. Our construction is defined as follows. 
	
By \cite[Lemma~6]{acharya2021differentially} , there exists a finite set $\cV\subset \mathbb{S}_2^{d-1}$, with cardinality $|\cV|\geq 2^{d}$, such that for any $v\neq v'\in \cV$, $\|v-v'\|\geq 1/2$.

Let $f_{(0, \mathbf{I}_d)}$ denotes the density function of $\cN(0, \mathbf{I}_d)$. Let $Q_v$ be a uniform distribution on two point masses $\{\pm \alpha^{-\frac{1}{4}} v\}$. Let $Q_0$ be Gaussian distribution $\cN(0, \mathbf{I}_d)$. For $\alpha\in (0,1]$, we construct $P_v:=(1-\alpha)Q_0+\alpha Q_v$. It is easy to see that $P_v$ is a distribution over $\reals^d$ with the following density function.
\begin{align}
	P_v(x)= \begin{cases}\frac{\alpha}{2}, & \text { if } x=-\alpha^{-\frac{1}{4}} v\;, \\ \frac{\alpha}{2}, & \text { if } x=\alpha^{-\frac{1}{4}} v\;, \\ (1-\alpha) f_{(0, \mathbf{I}_d)}(x) & \text { otherwise }\end{cases}\;.
\end{align}
The mean of $P_v$ is $0$. The covariance of $P_v$ is $\Sigma_v=(1-\alpha)\mathbf{I}_d+\sqrt{\alpha}vv^\top$. The top eigenvalue is $\lambda_1=1-\alpha+\sqrt{\alpha}$, the top eigenvector is $v$, and the second eigenvalue is $\lambda_2=1-\alpha$. And $\kappa=O(\alpha^{-1/2})$.

If $x=\alpha^{-1/4}v$, then $\|xx^\top-\Sigma_v\|_2 = O(\alpha^{-1/2})$. If $x\sim \cN(0, \mathbf{I}_{d})$,  we know $\|xx^\top -\Sigma_v\|_2= O(d)$. This implies $P_v$ satisfies \ref{asmp:A2} in Assumption~\ref{asmp:A} with $M=O((d+\alpha^{-1/2})\log(n))$ for $n$ i.i.d. samples.

It is easy to see that $\|\E[(xx^\top-\Sigma_v)(xx^\top-\Sigma_v)^\top]\|_2=O(d)$. This means $P_v$ satisfies
 \ref{asmp:A3} in Assumption~\ref{asmp:A} with $V=O(d)$.
 
By the fact that $\E[\ip{x}{u}^2]=O(1)$ and $\E[\ip{x}{u}^4]=O(1)$ for any unit vector $u$, we have $\gamma^2 = \|\E[(xx^\top-\Sigma_v)uu^\top(xx^\top-\Sigma_v)^\top]\|_2=O(1)$ for any unit vector $u$.

Our proof technique is based on following lemma. 
\begin{lemma}[{\cite[Theorem 3]{barber2014privacy}}]
\label{lemma:packing}
	  Fix $\alpha\in (0,1]$. Define $P_v=(1-\alpha)Q_0+\alpha Q_v$ for $v\in \cV$ such that such that  $\rho(\theta(P_v), \theta(P_{v'}))\geq 2t$. Let $\hat{\theta}$ be a $(\varepsilon, \delta)$ differentially private estimator. Then, 
	\begin{eqnarray}
		\frac{1}{|\mathcal{V}|} \sum_{\nu \in \mathcal{V}} P_{v}\left(\rho\left(\widehat{\theta}, \theta(P_v)\right) \geq t\right) \geq \frac{(|\mathcal{V}|-1) \cdot\left(\frac{1}{2} e^{-\varepsilon\lceil n \alpha\rceil}-\delta \frac{1-e^{-\varepsilon[n \alpha\rceil}}{1-e^{-\varepsilon}}\right)}{1+(|\mathcal{V}|-1) \cdot e^{-\varepsilon\lceil n \alpha\rceil}}\;.
	\end{eqnarray}
\end{lemma}
 
Set $\rho(\theta(P_v), \theta(P_{v'}))=\sin(v, v')/\kappa$. By Lemma~\ref{lem:sin2distance},  $\rho(\theta(P_v), \theta(P_{v'})) \geq \|v-v'\|/\kappa= \Omega(\sqrt{\alpha})$. 

Lemma~\ref{lemma:packing} implies
\begin{align}
	\sup_{P\in \tilde\cP}\E_{S \sim P^n}[\sin(\hat{v}(S), v_1(\Sigma))]
		&\geq  \frac{1}{|\cV|}\sum_{v\in \cV}\E_{S \sim P_v^n}[\sin(\hat{v}(S), v_1(\Sigma_v))]\\
		&= \kappa t\frac{1}{|\cV|}\sum_{v\in \cV}P_{v}\left(\frac{\sin(\hat{v}(S), v_1(\Sigma_v))}{\kappa}\geq t\right)\\
		&\gtrsim \kappa t \frac{(2^d-1) \cdot\left(\frac{1}{2} e^{-\varepsilon\lceil n \alpha\rceil}-\frac{\delta}{1-e^{-\varepsilon}}\right)}{1+(2^d-1) e^{-\varepsilon\lceil n \alpha\rceil}}\;,
\end{align}

For $d\geq 2$, we know $2^d-1\geq e^{d/2}$. We choose 
\begin{align}
	\alpha = \min \left\{\frac{1}{n \varepsilon}\left(\frac{d}{2}-\varepsilon\right), \frac{1}{n \varepsilon}\log \left(\frac{1-e^{-\varepsilon}}{4 \delta e^{\varepsilon}}\right), 1\right\}\;.
\end{align}

This implies 
\begin{align}
	\frac{1}{2} e^{-\varepsilon\lceil n \alpha\rceil}-\frac{\delta}{1-e^{-\varepsilon}} \geq \frac{1}{4} e^{-\varepsilon(n \alpha+1)}>0\;.
\end{align}

So we have there exists a constant $C$ such that
\begin{align}
	\inf_{\hat{v}}\sup_{P\in \tilde\cP}\E_{S\sim P^n}\left[\sin(\hat{v}(S), v_1(\Sigma))\right]
&\geq  C\kappa \sqrt{\alpha}\frac{\frac{1}{4} e^{d / 2} e^{-\varepsilon(n \alpha+1)}}{1+e^{d / 2} e^{-\varepsilon(n \alpha+1)}} \\
&\gtrsim\kappa \min\left(1,\sqrt{\frac{d \wedge \log \left(\left(1-e^{-\varepsilon}\right)/\delta\right)}{n\varepsilon}}\right)\;.
\end{align}

\subsection{Proof of Theorem~\ref{thm:lbht} }
\label{sec:lbht_proof}

	Similar to the proof of Theorem~\ref{thm:lb}, we use DP Fano's method in Theorem~\ref{thm:dp_fano}. It suffices to construct an indexed set $\cV$ and indexed distribution family $\cP_\cV$ such that $x_ix_i^\top$ satisfies \ref{asmp:A1}, \ref{asmp:A2} and \ref{asmp:A3} in Assumption~\ref{asmp:A}. Our construction is defined as follows. 
	
Let $\lambda_1>\lambda_2>0$. By Lemma~\ref{lem:packing_alpha}, there exists a finite set $\cV_\alpha\subset \mathbb{S}_2^{d-1}$, with cardinality $|\cV_\alpha|=2^{\Omega(d)}$, such that for any $v\neq v'\in \cV_\alpha$, $\alpha/\sqrt{2}\leq \|v-v'\|\leq \sqrt{2}$, where $\alpha:= \sqrt{\lambda_2/\lambda_1}$.

Let $f_{(0, S)}$ denotes the density function of $\cN(0, S)$. We construct $P_v$ over $\reals^d$ for $v\in \cV_\alpha$ with the following density function.
\begin{align}
	P_v(x)= \begin{cases}\frac{1-\lambda_2/\lambda_1}{2d}, & \text { if } x=-\sqrt{d\lambda_1} v\;, \\ \frac{1-\lambda_2/\lambda_1}{2d}, & \text { if } x=\sqrt{d\lambda_1} v\;, \\ 1-\frac{1-\lambda_2/\lambda_1}{d},  f_{(0, \frac{\lambda_2}{1-\frac{1-\lambda_2/\lambda_1}{d}}\mathbf{I}_{d})}(x) & \text { otherwise }\end{cases}\;.
\end{align}
The mean of $P_v$ is $0$. The covariance of $P_v$ is $\Sigma_v:= (\lambda_1-\lambda_2)vv^\top+\lambda_2\mathbf{I}_d$. It is easy to see that the top eigenvalue is $\lambda_1$, the top eigenvector is $v$, and the second eigenvalue is $\lambda_2$. 

If $x=\sqrt{d\lambda_1}v$, then $\|xx^\top-\Sigma_v\|_2=\|(d\lambda_1-\lambda_1+\lambda_2)-\lambda_2\mathbf{I}_d\|_2 = O(d\lambda_1)$. If $x\sim \cN(0, \frac{\lambda_2}{1-\frac{1-\lambda_2/\lambda_1}{d}}\mathbf{I}_{d})$, by the fact that $\frac{\lambda_2}{1-\frac{1-\lambda_2/\lambda_1}{d}}\leq \lambda_1$, we know $\|xx^\top -\Sigma_v\|_2\leq O(d\lambda_1)$. This implies $P_v$ satisfies \ref{asmp:A2} in Assumption~\ref{asmp:A} with $M=O(d \log(n))$ for $n$ i.i.d. samples.

Similarly, $\|\E[(xx^\top-\Sigma_v)(xx^\top-\Sigma_v)^\top]\|_2\leq \|d(\lambda_1^2-\lambda_1\lambda_2)vv^\top+d\lambda_2\lambda_1+3\Sigma_v\Sigma_v^\top\|_2=O(d\lambda_1^2)$. This means $P_v$ satisfies
 \ref{asmp:A3} in Assumption~\ref{asmp:A} with $V=O(d)$.
 
For $v\neq v'\in \cV_\alpha$, we have $D_{\rm TV}(P_v, P_{v'})=(1-\lambda_2/\lambda_1)/d$. By Lemma~\ref{lem:sin2distance},  $\sin(v, v')\geq \|v-v'\|/\sqrt{2}\geq (\sqrt{\lambda_2/\lambda_1})/2$.

By Theorem~\ref{thm:dp_fano}, there exists a constant $C$ such that
\begin{align}
	\inf_{\hat{v}}\sup_{P\in \cP_\Sigma}\E_{S\sim P^n}\left[\sin(\hat{v}(S), v_1(\Sigma))\right]
\geq  C\min\left(\sqrt{\frac{\lambda_2}{\lambda_1}}, \frac{d^2}{n\varepsilon}\sqrt{\frac{\lambda_1\lambda_2}{(\lambda_1-\lambda_2)^2}}\right)\;.
\end{align}

\section{The analysis of Private Oja's Algorithm}
\label{sec:dpoja_proof}
We analyze Private Oja's Algorithm in Algorithm~\ref{alg:dpoja}.

\subsection{Proof of privacy in Lemma~\ref{lem:dpoja_privacy}}
\label{sec:dpoja_privacy_proof}

We use following Theorem~\ref{thm:shuffling} to prove our privacy guarantees. 
\begin{theorem}[Privacy amplification by shuffling {\cite[Theorem~3.8]{feldman2022hiding}}]
\label{thm:shuffling}
	For any domain $\cD$, let $\cR^{(i)}:\cS^{(1)}\times \cdots \times \cS^{(i-1)}\times \cD\rightarrow \cS^{(i)}$ for $i\in [n]$ (where $\cS^{(i)}$ is the range space of $\cR^{(i)}$) be a sequence of algorithms such that $\cR^{(i)}(z_{1:i-1}, \cdot)$ is an $(\varepsilon_0, \delta_0)$-DP local randomizer for all values of auxiliary inputs $z_{1:i-1}\in \cS^{(1)}\times \cdots \times \cS^{(i-1)}$. Let $\cA_{S}: \cD^{n}\rightarrow\cS^{(1)}\times \cdots \times \cS^{(n)} $ be the algorithm that given a dataset $x_{1:n}\in \cD^{n}$, samples a uniform random permutation $\pi$ over $[n]$, then sequentially computes $z_i=\cR^{(i)}(z_{1:i-1}, x_{\pi(i)})$ for $i\in [n]$ and outputs $z_{1:n}$. Then for any $\delta\in [0, 1]$ such that $\varepsilon_0\leq \log\left(\frac{n}{16\log(2/\delta)}\right )$, $\cA_s$ is $(\varepsilon, \delta+O(e^{\varepsilon}\delta_0n))$-DP, where 
	\begin{align}
		\varepsilon =O  \left((1-e^{-\varepsilon_0})\left(\frac{ \sqrt{e^{\varepsilon_{0}} \log (1 / \delta)}}{\sqrt{n}}+\frac{ e^{\varepsilon_{0}}}{n}\right)\right)\;. 
	\end{align}
\end{theorem}

Let $\cR^{(t)}(w_{t-1}, A_{\pi(t)}):=w_t$. Let $\varepsilon_0=\frac{\sqrt{2\log(1.25/\delta_0)}}{\alpha}$.  We show $\cR^{(t)}(w_{t-1}, \cdot)$ is an $(\varepsilon_0, \delta_0)$-DP local randomizer.  

	If there is no noise in each update step, the update rule is 
	\begin{align}
		w_t'&\gets w_{t-1}+\eta_t{\rm clip}_\beta\left(A_t  w_{t-1}\right)\;,\\
		w_t &\gets w_{t-1}/\|w_{t-1}\|
	\end{align}
	The sensitivity of $w_t'$ is $2\beta\eta_t$ with respect to a difference in $A_t$. By Gaussian mechanism in Lemma~\ref{lem:gauss} and post processing property of local differential privacy, we know $w_t$ is $(\varepsilon_0, \delta_0)$-DP local randomizer.
	
	Assume that $\varepsilon_0=\frac{\sqrt{2\log(1.25/\delta_0)}}{\alpha}\leq \frac{1}{2}$. By Theorem~\ref{thm:shuffling}, for $\hat{\delta}\in [0,1]$ such that $\varepsilon_0\leq \log\left(\frac{n}{16\log(2/\hat\delta)}\right )$, Algorithm~\ref{alg:dpoja} is $(\hat\varepsilon, \hat\delta+O(e^{\hat\varepsilon}\delta_0n))$-DP  and for some constant $c_1>0$, 
	\begin{align}
		\hat\varepsilon &\leq c_1  \left((1-e^{-\varepsilon_0})\left(\frac{ \sqrt{e^{\varepsilon_{0}} \log (1 / \hat\delta)}}{\sqrt{n}}+\frac{ e^{\varepsilon_{0}}}{n}\right)\right)\\
		&\leq c_1  \left((e^{0.5}-e^{-0.5\varepsilon_0})\frac{ \sqrt{ \log (1 / \hat\delta)}}{\sqrt{n}}+\frac{ e^{\varepsilon_{0}}-1}{n}\right)\\
		&\leq c_1\left(((1+\varepsilon_0)-(1-\varepsilon_0/2))\frac{ \sqrt{ \log (1 / \hat\delta)}}{\sqrt{n}}+\frac{1+2\varepsilon_0-1}{n}\right)\\
		&=c_1\varepsilon_0\left(\frac{1}{2}\sqrt{\frac{\log(1/\hat\delta)}{n}}+\frac{2}{n}\right)\\
		&\leq c_2\frac{\sqrt{\log(1/\delta_0)}}{\alpha}\sqrt{\frac{\log(1/\hat\delta)}{n}}\;, 
	\end{align}
	for some absolute constant $c_2>0$.
	
	Set $\hat{\delta}=\delta/2$, $\delta_0=c_3 \delta/(e^{\hat\varepsilon}n)$ for some $c_3>0$ and $\alpha=C'\log(n/\delta)/(\varepsilon\sqrt{n})$. We have
	\begin{align}
		\hat{\varepsilon}&\leq c_2\frac{\sqrt{\log(e^{\hat\varepsilon}n/(c_3\delta))}}{\alpha}\sqrt{\frac{\log(2/\delta)}{n}}\\
		&= \frac{\sqrt{\log(e^{\hat\varepsilon}n/(c_3\delta))\log(2/\delta)}}{C'\log(n/\delta)}\cdot \varepsilon .\label{eq:hat_varepsilon}
	\end{align}
	
	For any $\varepsilon\leq 1$, by Eq.~\eqref{eq:hat_varepsilon}, there exists some sufficiently large $C'>0$ such that  $\hat{\varepsilon}\leq \varepsilon$. 
	
	Recall that we assume $\varepsilon_0=\frac{\sqrt{2\log(1.25/\delta_0)}}{\alpha}\leq \frac{1}{2}$. This means $\varepsilon =O(\sqrt{\frac{\log(n/\delta)}{n})}$. 
\subsection{Proof of clipping in  Lemma~\ref{lem:clipping_oja}}
\label{sec:lippping_oja_proof}

	Let $z_t = A_tw_{t-1}$. Let $\mu_t := \E[z_t]= \Sigma w_{t-1}$. By Lemma~\ref{lem:tail}, we know for any $\|v\|=1$, with probability $1-\zeta$, 
	\begin{align}
		|v^\top (z_t-\mu_t)|\leq K\gamma\lambda_1 \log^{a}(1/\zeta)\;.
	\end{align}
	
	Applying union bound over all basis vectors $v\in \{e_1,\ldots, e_d\}$ and all samples, we know with probability $1-\zeta$, for all $j\in [d]$ and $t\in [n]$
	\begin{align}
		|z_{t,j}|\leq K\gamma\lambda_1 \log^{a}(nd/\zeta)+\lambda_1\;.
	\end{align}
	
	This implies that with probability $1-\zeta$, for all $t\in [n]$, we have
	\begin{align}
		\|z_t\|\leq (K\gamma\log^{a}(nd/\zeta)+1)\lambda_1\sqrt{d}\;.
	\end{align}

\subsection{Proof of utility in Theorem~\ref{thm:dpoja}}
\label{sec:dpoja_main_proof}

Lemma~\ref{lem:clipping_oja} implies that with probability $1-O(\zeta)$, Algorithm~\ref{alg:dpoja} does not have any clipping. Under this event, the update rule becomes
\begin{align}
		w_{t}'&\gets  w_{t-1}+\eta_t \left( A_{t} +2\alpha\beta G_t\right)w_{t-1} \\
		w_t &\gets w_t'/\|w_t'\|\;,
	\end{align}
where $\beta=(K\gamma\log^{a}(nd/\zeta)+1)\lambda_1\sqrt{d}$ and each entry in $G_t\in \reals^{d\times d}$ is i.i.d. sampled from standard Gaussian $\cN(0,1)$. 
This follows form the fact that $\|w_{t-1}\|=1$ and $G_t w_{t-1} \sim {\cal N}(0, {\bf I}_d)$.

Let $B_t = A_t +2\alpha\beta G_t$.  We show $B_t$ satisfies the three conditions in Theorem~\ref{thm:non-private-oja} ({\cite[Theorem~4.12]{jain2016streaming}}). It is easy to see that $\E[B_t]=\Sigma$ from Assumption~\ref{asmp:A1}. Next, we show upper bound of  $\max\left\{\left\|\E\left[(B_t-\Sigma)(B_t-\Sigma)^\top\right]\right\|_2, \left\|\E\left[(B_t-\Sigma)^\top(B_t-\Sigma)\right]\right\|_2\right\} $. 
We have 
\begin{align}
	&\left\|\E\left[(B_t-\Sigma)(B_t-\Sigma)^\top\right]\right\|_2 \nonumber \\
	=&\;\left\|\E[(A_t +2\alpha\beta G_t-\Sigma)( A_{t} +2\alpha\beta G_t-\Sigma)^\top]\right\|_2 \nonumber \\
	\leq &\;\left\|\E[(A_t-\Sigma)(A_{t}-\Sigma)^\top]\right\|_2+4\alpha^2\beta^2\|\E[G_tG_t^\top]\|_2 \nonumber \\
	\leq  &\; V\lambda_1^2+4\alpha^2\beta^2C_2 d \;,
	\end{align}
where the last inequality follows from Lemma~\ref{lem:norm_noise} and  $C_2>0$ is an absolute constant. Let $\widetilde{V}:= V\lambda_1^2+4\alpha^2\beta^2C_2 d$.
 Similarly, we can show that $ \left\|\E\left[(B_t-\Sigma)^\top(B_t-\Sigma)\right]\right\|_2\leq \widetilde{V}$. 
 
By Lemma~\ref{lemma:gaussian_matrix_tail}, we know with probability $1-\zeta$, for all $t\in [T]$,
\begin{align*}
	&\left\|B_t-\Sigma\right\|_2\\
	=&\left\|A_t +2\alpha\beta G_t-\Sigma\right\|_2\\
	\leq & \left\|A_t -\Sigma\right\|_2+2\alpha\beta \|G_t\|_2\\
	\leq & M\lambda_1+ 2C_3\alpha\beta\left(\sqrt{d}+\sqrt{\log(n/\zeta)}\right)\;.
\end{align*}
Let $\widetilde{M} := M\lambda_1+ 2C_3\alpha\beta\left(\sqrt{d}+\sqrt{\log(n/\zeta)}\right)$. 

Under the event that $\left\|B_t-\Sigma\right\|_2\leq \widetilde{M}$ for all $t\in [n]$, we apply Theorem~\ref{thm:non-private-oja} with a learning rate $\eta_t=\frac{h}{(\lambda_1-\lambda_2)(\xi+t)}$ where 
\begin{align}
    \xi=20\max\left(\frac{\widetilde{M}h}{(\lambda_1-\lambda_2)}, \frac{\left(\widetilde{V}+\lambda_1^2\right)h^2}{(\lambda_1-\lambda_2)^2\log(1+\frac{\zeta}{100})}\right)\;.
\end{align}
Then Theorem~\ref{thm:non-private-oja} implies that with probability $1-\zeta$, 

\begin{align}
	\sin^{2}\left(w_n, v_{1}\right)  
	\leq  \frac{C\log(1/\zeta)}{\zeta^2}\left(d\left(\frac{\xi}{n}\right)^{2h}+\frac{h^2 \widetilde{V}}{(2h-1)\left(\lambda_{1}-\lambda_{2}\right)^{2}n} \right)\;,
\end{align}
for some positive constant $C$.

Set $\alpha=\frac{C'\log(n/\delta)}{\varepsilon\sqrt{n}}$, the above bound implies
\begin{align}
    \sin^{2}\left(w_n, v_{1}\right) \leq  \frac{C\log(1/\zeta)}{\zeta^2}\left(\frac{ h^2   V\lambda_1^2}{(2h-1)\left(\lambda_{1}-\lambda_{2}\right)^{2}n}+\frac{(K\gamma\log^a(nd/\zeta)+1)^2\lambda_1^2\log^2(n/\delta)d^2h^2}{(2h-1)(\lambda_1-\lambda_2)^2\varepsilon^2n^2}+d\left(\tilde{\xi}\right)^{h} \right)\;,
\end{align} 
where $\tilde{\xi}=(\xi/n)^2$, and 
\begin{align}
    \tilde{\xi}:=\max&\left(\frac{M^2\lambda_1^2h^2}{(\lambda_1-\lambda_2)^2n^2}+\frac{(K\gamma\log^a(nd/\zeta)+1)^2\lambda_1^2\log^3(n/\delta)h^2 d^2}{(\lambda_1-\lambda_2)^2\varepsilon^2n^3},  \right.\nonumber \\
    &\left. \frac{V^2\lambda_1^4h^4}{(\lambda_1-\lambda_2)^4\log^2(1+\frac{\zeta}{100}) n^2}+\frac{(K\gamma\log^a(nd/\zeta)+1)^4\lambda_1^4\log^4(n/\delta)h^4d^4}{(\lambda_1-\lambda_2)^4\log^2(1+\frac{\zeta}{100})\varepsilon^4n^4}
    \right. \nonumber \\
    &\left. + \frac{\lambda_1^4h^4}{(\lambda_1-\lambda_2)^4\log^2(1+\frac{\zeta}{100})n^2}
    \right)\;.
\end{align}

For $\zeta = O(1)$ and $K=O(1)$, 
selecting $h=c \log n$, and assuming 
\begin{align}
    n\geq &C\left(\frac{M\lambda_1\log(n)}{\lambda_1-\lambda_2}+\frac{(K\gamma\log^a(nd/\zeta)+1)^{2/3}\lambda_1^{2/3}\log(n/\delta)\log^{2/3}(n)d^{2/3}}{(\lambda_1-\lambda_2)^{2/3}\varepsilon^{2/3}}\right. \nonumber\\
    &\left.+\frac{V\lambda_1^2(\log(n))^2}{(\lambda_1-\lambda_2)^2}+\frac{(K\gamma\log^a(nd/\zeta)+1)\lambda_1\log(n/\delta)\log(n)d}{(\lambda_1-\lambda_2)\varepsilon}+\frac{\lambda_1^2\log^2(n)}{(\lambda_1-\lambda_2)^2}
    \right)\;,
\end{align}
with large enough positive constants $c$, and $C$, we have $\tilde \xi\leq 1$ and $d\tilde\xi^\alpha \leq 1/n^2$. Hence it is sufficient to have 
$$ n=\tilde O\Big(\,\frac{\lambda_1^2}{(\lambda_1-\lambda_2)^2} + \frac{M\lambda_1}{\lambda_1-\lambda_2} + \frac{V\lambda_1^2}{(\lambda_1-\lambda_2)^2} + \frac{d\,(\gamma+1)\lambda_1\,\log (1/\delta)}{(\lambda_1-\lambda_2)\varepsilon}\, \Big)\;,$$
with a large enough constant.

\section{The analysis of DP-PCA}

We provides the proofs for Theorem~\ref{thm:main}, Theorem~\ref{thm:eigen}, and Lemma~\ref{lem:dpmean} that guarantees the privacy and utility of DP-PCA.   

\subsection{Proof of Theorem~\ref{thm:main} on the privacy and utility of DP-PCA}
\label{sec:dppca_proof}

From Theorem~\ref{thm:eigen} we know that Alg.~\ref{alg:eigen} returns $\hat{\Lambda}$ satisfying $2\hat{\Lambda}\geq \lambda_1^2\|H_u\|_2 $ with high probability. Then, from Lemma~\ref{lem:dpmean}, we know that with high probability Alg~\ref{alg:dpmean} returns an unbiased estimate of the gradient mean with added Gaussian noise. 
Under this case,  the update rule becomes
\begin{align}
		w_{t}'&\gets  w_{t-1}+\eta_t \left(\frac{1}{B}\sum_{i=1}^B A_{B(t-1)+i} +\beta_t G_t\right)w_{t-1} \\
		w_t &\gets w_t'/\|w_t'\|\;,
	\end{align}
where $\beta_t=\frac{8K\sqrt{2\hat{\Lambda}_t}\log^a(Bd/\zeta)\sqrt{2d\log(2.5/\delta)}}{\varepsilon B}$, $\hat{\Lambda}_t$ denote the estimated eigenvalue of covariance of the gradients at $t$-th iteration, and each entry in $G_t\in \reals^{d\times d}$ is i.i.d. sampled from standard Gaussian $\cN(0,1)$. 
This follows form the fact that $\|w_{t-1}\|=1$ and $G_t w_{t-1} \sim {\cal N}(0, {\bf I}_d)$.

Let $\beta := \frac{16K \gamma\lambda_1 \log^a(Bd/\zeta)\sqrt{2d\log(2.5/\delta)}}{\varepsilon B}$ such that $\beta \geq \beta_t$, which follows from the fact that $\hat{\Lambda} \leq  \sqrt{2} \lambda_1^2\|H_u\|_2 \leq \sqrt{2} \lambda_1^2\gamma^2$ (Theorem~\ref{thm:eigen} and Assumption~\ref{asmp:A4}).
Let $B_t = (1/B) \sum_{i=1}^B A_{B(t-1)+i} +\beta_t G_t$.  We show $B_t$ satisfies the three conditions in Theorem~\ref{thm:non-private-oja} ({\cite[Theorem~4.12]{jain2016streaming}}). It is easy to see that $\E[B_t]=\Sigma$ from Assumption~\ref{asmp:A1}. Next, we show upper bound of  $\max\left\{\left\|\E\left[(B_t-\Sigma)(B_t-\Sigma)^\top\right]\right\|_2, \left\|\E\left[(B_t-\Sigma)^\top(B_t-\Sigma)\right]\right\|_2\right\} $. 
We have 
\begin{align}
	&\left\|\E\left[(B_t-\Sigma)(B_t-\Sigma)^\top\right]\right\|_2 \nonumber \\
	=&\;\left\|\E[(\frac{1}{B}\sum_{i=1}^B A_{B(t-1)+i} +\beta_t G_t-\Sigma)(\frac{1}{B}\sum_{i=1}^B A_{B(t-1)+i} +\beta_t G_t-\Sigma)^\top]\right\|_2 \nonumber \\
	\leq &\;\left\|\E[(\frac{1}{B}\sum_{i=1}^B A_{B(t-1)+i}-\Sigma)(\frac{1}{B}\sum_{i=1}^B A_{B(t-1)+i}-\Sigma)^\top]\right\|_2+\beta^2\|\E[G_tG_t^\top]\|_2 \nonumber \\
	= &\; V\lambda_1^2/B+\beta^2\|\E[G_tG_t^\top]\|_2 \nonumber\\
	\leq &\; V\lambda_1^2/B+\beta^2C_2d\;, \label{eq:newV}
\end{align}
where the last inequality follows from Lemma~\ref{lem:norm_noise} and  $C_2>0$ is an absolute constant. 
Let $\widetilde{V}:= V\lambda_1^2/B+\beta^2C_2d$. Similarly, we can show that $ \left\|\E\left[(B_t-\Sigma)^\top(B_t-\Sigma)\right]\right\|_2\leq \widetilde{V}$. 
By Lemma~\ref{lem:bernstein} and Lemma~\ref{lemma:gaussian_matrix_tail}, we know with probability $1-\zeta$, for all $t\in [T]$,
\begin{align*}
	&\left\|B_t-\Sigma\right\|_2\\
	=&\left\|\frac{1}{B}\sum_{i=1}^B  A_{B(t-1)+i} +\beta_t G_t-\Sigma\right\|_2\\
	\leq & C_3\left( \frac{M\lambda_1\log (dT/\zeta)}{B}+\sqrt{\frac{V\lambda_1^2\log(dT/\zeta)}{B}}+\beta\left(\sqrt{d}+\sqrt{\log(T/\zeta)}\right)\right)\;.
\end{align*}
Let $\widetilde{M} := C_3\left( \frac{M\lambda_1\log (dT/\zeta)}{B}+\sqrt{\frac{V\lambda_1^2\log(dT/\zeta)}{B}}+\beta\left(\sqrt{d}+\sqrt{\log(T/\zeta)}\right)\right)$. 
Under the event that $\left\|B_t-\Sigma\right\|_2\leq \widetilde{M}$ for all $t\in [T]$, we apply Theorem~\ref{thm:non-private-oja} with a learning rate $\eta_t=\frac{\alpha}{(\lambda_1-\lambda_2)(\xi+t)}$ where 
\begin{align}
    \xi=20\max\left(\frac{\widetilde{M}\alpha}{(\lambda_1-\lambda_2)}, \frac{\left(\widetilde{V}+\lambda_1^2\right)\alpha^2}{(\lambda_1-\lambda_2)^2\log(1+\frac{\zeta}{100})}\right)\;.
    \label{eq:def_lr}
\end{align}
Then Theorem~\ref{thm:non-private-oja} implies that with probability $1-\zeta$, 

\begin{align}
	\sin^{2}\left(w_T, v_{1}\right)  
	\leq  \frac{C\log(1/\zeta)}{\zeta^2}\left(d\left(\frac{\xi}{T}\right)^{2\alpha}+\frac{\alpha^2 \widetilde{V}}{(2\alpha-1)\left(\lambda_{1}-\lambda_{2}\right)^{2}T} \right)\;,
\end{align}
for some positive constant $C$. 
Using $n=BT$ and Eq.~\eqref{eq:newV}, the above bound implies
\begin{align}
    \sin^{2}\left(w_T, v_{1}\right) \leq  \frac{C\log(1/\zeta)}{\zeta^2}\left(\frac{ \alpha^2 V\lambda_1^2}{(2\alpha-1)\left(\lambda_{1}-\lambda_{2}\right)^{2}n}+\frac{K^2\gamma^2\lambda_1^2\log^{2a}(nd/(T\zeta))\log(1/\delta)d^2\alpha^2T}{(2\alpha-1)(\lambda_1-\lambda_2)^2\varepsilon^2n^2}+d\left(\tilde{\xi}\right)^{\alpha} \right)\;. \label{eq:balance1}
\end{align} 
where $\tilde{\xi}=(\xi/T)^2$, and 
\begin{align}
    \tilde{\xi}:=\max&\left(\frac{M^2\lambda_1^2\alpha^2\log^2(dT/\zeta)}{(\lambda_1-\lambda_2)^2n^2}+\frac{V\lambda_1^2\log(dT/\zeta)\alpha^2}{(\lambda_1-\lambda_2)^2nT}+\frac{K^2\gamma^2\lambda_1^2\log^{2a}(nd/(T\zeta))\log(1/\delta)\log(T/\zeta)\alpha^2 d^2}{(\lambda_1-\lambda_2)^2\varepsilon^2n^2},  \right.\nonumber \\
    &\left. \frac{V^2\lambda_1^4\alpha^4}{(\lambda_1-\lambda_2)^4\log^2(1+\frac{\zeta}{100}) n^2}+\frac{K^4\gamma^4\lambda_1^4\log^{4a}(nd/(T\zeta))\log^2(1/\delta)\alpha^4d^4T^2}{(\lambda_1-\lambda_2)^4\log^2(1+\frac{\zeta}{100})\varepsilon^4n^4}
    \right. \nonumber \\
    &\left. + \frac{\lambda_1^4\alpha^4}{(\lambda_1-\lambda_2)^4\log^2(1+\frac{\zeta}{100})T^2}
    \right)\;. \label{eq:balance2}
\end{align}
For $\zeta = O(1)$ and $K=O(1)$, 
selecting $\alpha=c \log n$, $T=c'(\log n)^2$, and assuming $\log n \geq \lambda_1^2/(\lambda_1-\lambda_2)^2$ and 
\begin{align}
    n\geq &C\left(\frac{M\lambda_1\log(n)\log(d\log(n))}{\lambda_1-\lambda_2}+\frac{\sqrt{V\lambda_1^2\log(dT)}}{(\lambda_1-\lambda_2)}+\frac{\gamma\lambda_1\log^a(nd/\log(n))\sqrt{\log(1/\delta)\log(\log(n))}\log(n)d}{(\lambda_1-\lambda_2)\varepsilon}\right. \nonumber\\
    &\left.+\frac{V\lambda_1^2(\log(n))^2}{(\lambda_1-\lambda_2)^2}+\frac{\gamma\lambda_1\log^a(nd/\log(n))\sqrt{\log(1/\delta)}(\log(n))^2d}{(\lambda_1-\lambda_2)\varepsilon}
    \right)\;,
\end{align}
with large enough positive constants $c$, $c'$, and $C$, we have $\tilde \xi\leq 1$ and $d\tilde\xi^\alpha \leq 1/n^2$. Hence it is sufficient to have 
$$ n=\tilde O\Big(\,\exp(\lambda_1^2/(\lambda_1-\lambda_2)^2) + \frac{M\lambda_1}{\lambda_1-\lambda_2} + \frac{V\lambda_1^2}{(\lambda_1-\lambda_2)^2} + \frac{d\,\gamma\,\lambda_1\sqrt{\log (1/\delta)}}{(\lambda_1-\lambda_2)\varepsilon}\, \Big)\;,$$
with a large enough constant. 


\subsection{Algorithm and proof of Theorem~\ref{thm:eigen} on top eigenvalue estimation}
\label{sec:eigen_proof}

\begin{algorithm2e}[H]    
   \caption{Private Top Eigenvalue Estimation} 
   \label{alg:eigen} 
   	\DontPrintSemicolon 
	\KwIn{$S=\{g_i\}_{i=1}^B$, $(\varepsilon,\delta)$-DP, failure probability $\zeta$}
	\SetKwProg{Fn}{}{:}{}
	{ 
	Let $\tilde g_i \gets g_{2i} - g_{2i-1}$ for $i\in 1, 2, \ldots, \lfloor B/2\rfloor$. Let $\tilde{S}= \{\tilde g_i\}_{i=1}^{\lfloor B/2\rfloor}$

	Partition $\tilde{S}$ into $k={C_1\log(1/(\delta \zeta))}/{\varepsilon}$ subsets and denote each dataset as $G_j\in \reals^{d\times b}$, where each dataset is of size $b=\lfloor B/2k \rfloor$\\

		Let  $\lambda_1^{(j)}$ be the top eigenvalue of $(1/b)G_jG_j^\top$ for $\forall j \in [k]$ \\
	Partition $[0, \infty)$ into $\Omega\gets \left\{\ldots,\left[ 2^{-2/4}, 2^{-1 / 4}\right)\left[ 2^{-1 / 4}, 1\right)\left[1,2^{1 / 4}\right),\left[2^{1 / 4}, 2^{2/4}\right), \ldots\right\} \cup\{[0,0]\}$\\
	Run $(\varepsilon, \delta)$-DP histogram learner of Lemma~\ref{lem:hist-KV17} on $\{ \lambda_1^{(j)}\}_{j=1}^k$ over $\Omega$ \\
	{\bf if} all the bins are empty {\bf then} 
	Return $\perp$\\
	Let $[l, r]$ be a non-empty bin that contains the maximum number of points in the DP histogram\\
	Return  $\hat{\Lambda}= l$
	} 
\end{algorithm2e}

Taking the difference ensures that $\tilde g_i$ is zero mean, such that we can directly use the top eigenvalue of $(1/b)G_jG_j^\top$ for $j\in[k]$. 
We compute a histogram over those $k$ top eigenvalues. This histogram is privatized by adding noise only to the occupied bins and thresholding small entries of the histogram to be zero. The choice $k=\Omega(\log(1/\zeta)/\varepsilon)$ ensures that the most occupied bin does not change after adding the DP noise to the histograms, and $k=\Omega(\log(1/\delta)/\varepsilon)$ is necessary for handling unbounded number of bins.
We emphasize that  we do not require any upper and lower bounds on the eigenvalue, thanks to the private histogram learner from  \cite{bun2019simultaneous,karwa2018finite} that gracefully handles unbounded number of bins. 

The privacy guarantee follows from the privacy guarantee of the histogram learner provided in Lemma~\ref{lem:hist-KV17}.

    For utility analysis, we follow the analysis of {\cite[Theorem~3.1]{kamath2021private}}. The main difference is that  we prove a smaller sample complexity sine we only need the top eigenvalue, and we analyze a more general distribution family. 
    The random vector $\tilde g_i$ is zero mean with covariance $2\lambda_1^2H_u\in{\mathbb R}^{d\times d}$, where $H_u = \E[(A_i-\Sigma)uu^\top(A_i-\Sigma)^\top]/\lambda_1^2$, and 	$\tilde g_i$ satisfies with probability $1-\zeta$, 
	\begin{align}
		|\ip{\tilde g_i}{v}| \;\leq\; 2K\lambda_1\sqrt{\|H_u\|_2}\log^{a}(1/\zeta)\;, \label{eq:tail_y}
	\end{align}
	which follows from Lemma~\ref{lem:tail}. 
	Applying union bound over all basis vectors $v\in \{e_1,\ldots, e_d\} $, we know with probability $1-\zeta$,
		\begin{align*}
		\| \tilde g_i \| 
	 \; \leq \;  2K\lambda_1\sqrt{d\|H_u\|_2}\log^{a}(d/\zeta)\;.
	\end{align*}
	We next show that conditioned on event $\cE=\{\| \tilde g_i \|\leq 2K\lambda_1\sqrt{d\|H_u\|_2}\log^{a}(d/\zeta)\}$, the covariance $\E[\tilde g_i \tilde g_i ^\top|\cE]$ is close to the true covariance $\E[\tilde g_i \tilde g_i^\top]=2\lambda_1^2 H_u$. 
	Note that 
	\begin{align}
		\E[\tilde g_i \tilde g_i^\top|\cE] &\;=\; \frac{\E[ \tilde g_i \tilde g_i^\top \ind\{\|\tilde g_i\|\leq 2K\lambda_1\sqrt{d\|H_u\|_2}\log^{a}(d/\zeta)\}]}{\prob(\cE)} \nonumber 
		\\
		&\;\preceq\; \frac{\E[\tilde g_i \tilde g_i^\top ]}{\prob(\cE)} 
		\;\preceq\; \frac{2\lambda_1^2H_u}{1-\zeta}\;. \label{eq:cov_ub}
	\end{align} 
	
	
We next show the empirical covariance $({1}/{b})\sum_{i=1}^b\tilde g_i \tilde g_i^\top$ concentrates around $2\lambda_1^2H_u$. 
First of all, using union bound on Eq.~\eqref{eq:tail_y}, we have with probability $1-\zeta$, for all $i\in [b]$ and $j\in [d]$, 
\begin{align*}
	|\tilde{g}_{ij}|\leq 2K\lambda_1\sqrt{\|H_u\|_2}\log^{a}(bd/\zeta)\;.
\end{align*}
Under the event that $|\tilde{g}_{ij}|\leq 2K\lambda_1\sqrt{\|H_u\|_2}\log^{a}(nd/\zeta)$ for all $i\in [b]$, $j\in [d]$, {\cite[Corrollary~6.20]{wainwright2019high}} together with Eq.~\eqref{eq:cov_ub} implies 
	\begin{align*}
		\prob\left(\left\|\frac1b\sum_{i=1}^b \tilde g_i \tilde g_i^\top- 2\lambda_1^2H_u\right\|_2\geq \alpha\right)\leq 2d\exp\left(-\frac{b \alpha^2}{8K^2\lambda_1^2\|H_u\|_2\log^{2a}(\frac{bd}{\zeta})d(2\lambda_1^2\|H_u\|_2/(1-\zeta)+\alpha)}\right)\;.
	\end{align*} 
	The above bound implies that with probability $1-\zeta$,  
	\begin{align*}
		\left\|\frac1b\sum_{i=1}^b \tilde{g}_i\tilde{g}_i^\top-\lambda_1^2 2H_u\right\|_2 \;=\; O\Big(\, K\lambda_1^2\|H_u\|_2\log^a(bd/\zeta) \sqrt{\frac{d\log(d/\zeta)}{b}}+K^2\lambda_1^2\|H_u\|_2\log^{2a}(bd/\zeta)\frac{d\log(d/\zeta)}{b}
		\, \Big)\;.
	\end{align*}

	 This means if $b={\Omega}(K^2 d \log(dk/\zeta)\log^{2a}(bdk/\zeta))$, then with probability $1-\zeta$, for all $j\in [k]$,  $(1-2^{1/8})\lambda_1^2\|H_u\|_2\leq \lambda_1^{(j)}\leq (1+2^{1/8})\lambda_1^2\|H_u\|_2$. This means all of $\lambda_1^{(j)}$ must be within $2^{1/4}\lambda_1^2\|H_u\|_2$ interval. Thus, at most two consecutive buckets are filled with $\lambda_1^{(j)}$. By private histogram from Lemma~\ref{lem:hist-KV17}, if $k\geq \log(1/(\delta\zeta))/\varepsilon$, one of those two bins are released.  The resulting  total multiplicative error is bounded by  $2^{1/2}$.

\subsection{Algorithm and proof of Lemma~\ref{lem:dpmean} on DP mean estimation} 
\label{sec:dpmean_proof}

\begin{algorithm2e}[H]    
   \caption{Private Mean Estimation \cite{karwa2018finite,kamath2019privately} }
   \label{alg:dpmean} 
   	\DontPrintSemicolon 
	\KwIn{$S=\{g_i\}_{i=1}^B$, $(\varepsilon , \delta )$, target error $\alpha$, failure probability $\zeta$, approximate top eigenvalue $\hat{\Lambda}$ }
	\SetKwProg{Fn}{}{:}{}
	{ 
	
	Let $\tau=2^{1/4}K\sqrt{\hat{\Lambda}}\log^a(25)$.\\
	\For{j=1, 2, \ldots, d}{

	Run $(\frac{\varepsilon}{4\sqrt{2d\log(4/\delta)}}, \frac{\delta}{4d})$-DP histogram learner of Lemma~\ref{lem:hist-KV17} on $\{g_{ij}\}_{i\in [B]}$ over $\Omega=\{\cdots,(-2 \tau,- \tau],(- \tau, 0],(0, \tau],( \tau, 2 \tau],(2 \tau, 3 \tau] \cdots\}$.\\
	Let $[l, h]$ be the bucket that contains maximum number of points in the private histogram\\ 
	 $\bar{g}_j\gets l$\\
	Truncate the $j$-th coordinate of gradient $\{g_i\}_{i\in [B]}$ by $[\bar{g}_j-3K\sqrt{\hat{\Lambda}}\log^a(Bd/\zeta), \bar{g}_j+3K\sqrt{\hat{\Lambda}}\log^a(Bd/\zeta)]$.\\
	Let $\tilde{g}_i$ be the truncated version of $g_i$.
	}
	Compute empirical mean of truncated gradients  $\tilde{\mu}= (1/B) \sum_{i=1}^B\tilde{g}_i$ and add Gaussian noise: $\hat{\mu}=\tilde{\mu}+\cN\left(0, \left(\frac{12K\sqrt{\hat{\Lambda}}\log^a(Bd/\zeta)\sqrt{2d\log(2.5/\delta)}}{\varepsilon B}\right)^2\mathbf{I}_d\right)$\\
	Return $\hat{\mu}$
	} 
\end{algorithm2e}

    The histogram learner is called $d$ times, each with $(\varepsilon/(4\sqrt{2d\log(4/\delta)}),\delta/(4d))$-DP guarantee, and the end-to-end privacy guarantee is $(\varepsilon/2,\delta/2)$ from Lemma~\ref{lem:composition} for $\varepsilon\in(0,0.9)$. The sensitivity of the clipped mean estimate is $\Delta=\sqrt{d} 6K\sqrt{\hat{\Lambda}}\log^a(Bd/\zeta)$. Gaussian mechanism  with covariance $(2\Delta\sqrt{2\log(2.5/\delta)}/\varepsilon)^2{\bf I}_d$ satisfy $(\varepsilon/2,\delta/2)$-DP from Lemma~\ref{lem:gauss} for $\varepsilon \in (0,1)$. Putting these two together, with serial composition of Lemma~\ref{lem:serial}, we get the desired privacy guarantee.

    The proof of utility follows similarly as \cite[Lemma~D.2]{liu2021robust}. Let $I_l=(l\sqrt{\hat{\Lambda}}, (l+1)\sqrt{\hat{\Lambda}}]$. Denote the population probability of $j$-th coordinate at $I_l$ as $h_{j,l}=\prob(g_{ij}\in I_l)$. Denote the empirical probability as $\hat{h}_{j,l}=\frac{1}{B}\sum_{i=1}^B\ind(g_{ij}\in I_{l})$.  Denote the private empirical probability being released as $\tilde{h}_{j, l}$.

    Fix $j\in [d]$. Let $I_k$ be the bin that contains the $\mu_j$. Then we know $[\mu_j-K\lambda_1\sqrt{\|H_u\|_2}\log^a(25), \mu_j+K\lambda_1\sqrt{\|H_u\|_2}\log^a(25)]\subseteq [\mu_j-\tau, \mu_j+\tau]\subset(I_{k-1}\cup I_k \cup I_{k+1})$.     By Lemma~\ref{lem:tail}, we know $\prob(|g_{ij}-\mu_j|\geq \tau)\leq \prob(|g_{ij}-\mu_j|\geq K\lambda_1\sqrt{\|H_u\|_2}\log^a(25))\leq 0.04$. This means $h_{(k-1),j}+h_{k,j}+h_{(k+1),j}\geq 0.96$ and $\min(h_{(k-1),j},h_{k,j},h_{(k+1),j})\geq 0.32$.
    
    By Dvoretzky-Kiefer-Wolfowitz inequality and an union bound over $j\in  [d]$, we have that with probability $1-\zeta$, $\max_{j,l}|h_{j, l}-\hat{h}_{j, l}|\leq \sqrt{\log(d/\zeta)/B}$.   Using Lemma~\ref{lem:hist-KV17},  if $B=\Omega((\sqrt{d\log(1/\delta)}/\varepsilon)\log(d/(\zeta\delta)))$, with probability $1-\zeta$, we have $\max_{j,l}|\tilde{h}_{j, l}-\hat{h}_{j, l}|\leq 0.005$. Thus, with our assumption on $B$, we can make sure with probability $1-\zeta$, $\max_{j,l}|\tilde{h}_{j, l}-h_{j, l}|\leq 0.01$. Then we have $\min(h_{(k-1),j},h_{k,j},h_{(k+1),j})-0.01\geq 0.31\geq 0.04+0.01\geq \max_{l\neq k-1, k, k_1}h_{j, l}+0.01$. This implies with probability $1-\zeta$, the algorithm must pick one of the bins from $I_{k-1}, I_k, I_{k+1}$. This means $|\bar{g}_{j}-\mu_j|\leq 2\tau \leq 2^{1.5}K\lambda_1\sqrt{\|H_u\|_2}\log^a(25)$. By tail bound of Lemma~\ref{lem:tail}, we know for all $j\in [d]$ and $i\in [B]$, $|g_{ij}-\bar{g}_j|\leq |g_{ij}-\mu_j|+|\bar{g}_{j}-\mu_j|\leq  3K\lambda_1\sqrt{\|H_u\|_2}\log^a(Bd/\zeta)$.  This completes our proof.



\section{Technical lemmas}

\begin{lemma}
\label{lemma:gaussian_norm}
	Let $x\in \reals^d\sim \cN(0, \Sigma)$. Then there exists universal constant $C$ such that with probability $1-\zeta$,
	\begin{align}
		\|x\|^2\leq C \Tr(\Sigma)\log(1/\zeta)\;.
	\end{align}
\end{lemma}

\begin{proof}
	Let $\tilde{x}:=\Sigma^{-1/2}x$. Then $\tilde{x}$ is also a Gaussian with $\tilde{x}\sim \cN(0, \mathbf{I}_d)$. By Hanson-Wright inequality ( {\cite[Theorem~6.2.1]{vershynin2018high}}), there exists universal constant $c>0$ such that with probability $1-\zeta$, 
\begin{align}
   \|x\|^2 = \tilde{x}^\top\Sigma \tilde{x}\leq \Tr(\Sigma)+ c(\|\Sigma\|_{\F}+\|\Sigma\|_2)\log(2/\zeta)\leq C \Tr(\Sigma)\log(1/\zeta)\;.
\end{align}

\end{proof}

\begin{lemma}[{\cite[Theorem~4.4.5]{vershynin2018high}}]
\label{lemma:gaussian_matrix_tail}
	Let $G\in \reals^{d\times d}$ be a random matrix where each entry $G_{ij}$ is i.i.d. sampled from standard Gaussian $\cN(0,1)$. Then there exists universal constant $C>0$ such that with probability $1-2e^{-t^2}$, $\|G\|_2\leq C(\sqrt{d}+t)$ for $t>0$.
\end{lemma}

\begin{lemma}
\label{lem:norm_noise}
Let $G\in \reals^{d\times d}$ be a random matrix where each entry $G_{ij}$ is i.i.d. sampled from standard Gaussian $\cN(0,1)$. Then we have
$\|\E[GG^\top]\|_2 \leq C_2 d$ and $\|\E[G^\top G]\|_2\leq C_2 d$.
\end{lemma}
\begin{proof}
    By Lemma~\ref{lemma:gaussian_matrix_tail}, there exists universal constant $C_3>0$ such that
	\begin{align}
		\prob\left(\|G\|\geq C_1(\sqrt{d}+s)\right)\leq e^{-s^2}, \;\;\;\forall s>0\;.
	\end{align}
	
	Then 
	\begin{align}
		\|\E[GG^\top]\|_2&\leq \E[\|GG^\top\|_2]\\
		&\leq \E[\|G\|_2^2]\\
		&=\int_{0}^{\infty} 2r\prob(\|G\|_2\geq r)  dr \leq C_1d+C_3\int_{\sqrt{d}}^{\infty} 2re^{-\frac{(r-\sqrt{d})^2}{2}} d\\
		&= C_1(d+\sqrt{2\pi d}+2)\leq C_2d \;,
	\end{align}
	where $C_2$ is an absolute constant.
	The proof for the second claim follows similarly.
\end{proof}

\begin{lemma}\label{lem:sin2distance}
	Let $x,y \in \mathbb{S}_2^{d-1}$. Then 
	\begin{align}
		1-\ip{x}{y}^2\leq \|x-y\|^2\;.
	\end{align}
	If $\|x-y\|^2\leq \sqrt{2}$, then 
	\begin{align}
		1-\ip{x}{y}^2\geq \frac{1}{2}\|x-y\|^2 \;.
	\end{align}
\end{lemma}

The following lemma follows from
matrix Bernstein inequality \cite{tropp2012user}.

\begin{lemma}
    \label{lem:bernstein}
	 Under \ref{asmp:A1}, \ref{asmp:A2}, and \ref{asmp:A3},  in Assumption~\ref{asmp:A},  with probability $1-\zeta$, 
	 \begin{align}
	 	\Big \|\frac{1}{B}\sum_{i\in [B]}A_i-\Sigma\Big\|_2
	 	\;=\;O\Big(\, \sqrt{\frac{\lambda_1^2 V\log(d/\zeta)}{B}}+\frac{\lambda_1 M\log(d/\zeta)}{B}\,\Big) \;.
	 \end{align}
\end{lemma}


\end{document}